\documentclass[imslayout,noinfoline]{imsart}

\RequirePackage[OT1]{fontenc}
\RequirePackage{amsthm,amsmath,amssymb,subcaption,graphicx,epstopdf,enumerate,lmodern}
\RequirePackage[numbers]{natbib}
\RequirePackage[colorlinks,citecolor=blue,urlcolor=blue]{hyperref}

\usepackage{amsmath}
\usepackage{amssymb}
\usepackage{amscd}
\usepackage{bbm}

\font\tenmsb=msbm10

\def\eps{\varepsilon}
\font\tencmmib=cmmib10 \skewchar\tencmmib '60
\newfam\cmmibfam
\textfont\cmmibfam=\tencmmib

\font\tenmsb=msbm10

\def\Bbb#1{\hbox{\tenmsb#1}}

\def\bbox{\quad\hbox{\vrule \vbox{\hrule \vskip2pt \hbox{\hskip2pt
\vbox{\hsize=1pt}\hskip2pt} \vskip2pt\hrule}\vrule}}
\def\lessim{\ \lower4pt\hbox{$
\buildrel{\displaystyle <}\over\sim$}\ }
\def\gessim{\ \lower4pt\hbox{$\buildrel{\displaystyle >}
\over\sim$}\ }

\def\eps{{\varepsilon}}

\def\Bbb E{\mathbb{E}}
\def\Bbb R{\mathbb{R}} 
\newtheorem{remark}{Remark}
\newtheorem{assumption}{Assumption}
\newtheorem{proposition}{Proposition}
\newtheorem{lemma}{Lemma}
\newtheorem{theorem}{Theorem}

\makeatletter
\@addtoreset{equation}{section}

\makeatother

\font\tencmmib=cmmib10 \skewchar\tencmmib '60
\newfam\cmmibfam
\textfont\cmmibfam=\tencmmib

\font\tenmsb=msbm10

\def\Bbb#1{\hbox{\tenmsb#1}}

\def\bbox{\quad\hbox{\vrule \vbox{\hrule \vskip2pt \hbox{\hskip2pt
\vbox{\hsize=1pt}\hskip2pt} \vskip2pt\hrule}\vrule}}
\def\lessim{\ \lower4pt\hbox{$
\buildrel{\displaystyle <}\over\sim$}\ }
\def\gessim{\ \lower4pt\hbox{$\buildrel{\displaystyle >}
\over\sim$}\ }

\def\eps{\varepsilon}

\def\go0{\to 0}

\def\leftitem#1{\item{\hbox to\parindent{\enspace#1\hfill}}}

\def\qed{{$\hfill \bbox$}}

\def\sg{\sigma}

\def\sg2{\sigma^2}

\def\__{_{\infty}}

\begin{document}

\begin{frontmatter}
\title{Estimation of low rank density matrices: bounds in Schatten norms and other distances}
\runtitle{Estimation of density matrices}
\runauthor{D. Xia and V. Koltchinskii}
\begin{aug}
\author{\fnms{Dong} \snm{Xia} \thanksref{t1}\ead[label=e2]{dxia7@math.gatech.edu}}
and
\author{\fnms{Vladimir} \snm{Koltchinskii} \thanksref{t2}\ead[label=e1]{vlad@math.gatech.edu}}
\address{School of Mathematics, Georgia Institute of Technology, Atlanta, GA, 30332, USA.\\
\printead{e1,e2}}
\thankstext{t1}{Supported in part by NSF Grants DMS-1207808 and CCF-1523768}
\thankstext{t2}{Supported in part by NSF Grants DMS-1509739 and CCF-1523768}
\affiliation{School of Mathematics\\Georgia Institute of Technology}
\end{aug}

\begin{abstract}
Let ${\mathcal S}_m$ be the set of all $m\times m$ density matrices 
(Hermitian positively semi-definite matrices of unit trace). Consider 
a problem of estimation of an unknown density matrix $\rho\in {\mathcal S}_m$
based on outcomes of $n$ measurements of observables $X_1,\dots, X_n\in {\mathbb H}_m$ (${\mathbb H}_m$ being the space of $m\times m$ Hermitian matrices)
for a quantum system identically prepared $n$ times in state $\rho.$ 
Outcomes $Y_1,\dots, Y_n$ of such measurements could be described by a trace regression model in which 
${\mathbb E}_{\rho}(Y_j|X_j)={\rm tr}(\rho X_j), j=1,\dots, n.$
The design variables $X_1,\dots, X_n$ are often sampled at random from the uniform 
distribution in an orthonormal basis $\{E_1,\dots, E_{m^2}\}$ of ${\mathbb H}_m$
(such as Pauli basis). 
The goal is to estimate the unknown density matrix $\rho$ based on the data 
$(X_1,Y_1), \dots, (X_n,Y_n).$ Let 
$$
\hat Z:=\frac{m^2}{n}\sum_{j=1}^n Y_j X_j
$$ 
and let $\check \rho$ be the projection of $\hat Z$
onto the convex set ${\mathcal S}_m$ of density matrices. 
It is shown that for estimator $\check \rho$ the minimax lower bounds in classes of low rank
density matrices (established earlier) are attained up logarithmic factors for all 
Schatten $p$-norm distances, $p\in [1,\infty]$ and for Bures version of quantum Hellinger 
distance. Moreover, for a slightly modified version of estimator $\check \rho$
the same property holds also for quantum relative entropy (Kullback-Leibler) distance
between density matrices.
\end{abstract}
\end{frontmatter}

\section{Introduction}

Let ${\mathbb M}_m$ be the set of all $m\times m$ matrices with entries in ${\mathbb C}.$
For $A\in {\mathbb M}_m,$ let $A^{\ast}$ denote its conjugate transpose and let ${\rm tr}(A)$ denote 
the trace of $A.$ The complex linear space ${\mathbb M}_m$ of dimension $m^2$ will be equipped with the Hilbert-Schmidt inner product 
$
\langle A,B\rangle = {\rm tr}(AB^{\ast}), A,B\in {\mathbb M}_m.
$
In what follows, the sign $\otimes $ denotes the tensor product of vectors or matrices (linear transformations).
For instance, for $u,v\in {\mathbb C}^m,$ $u\otimes v$ is a linear transformation from ${\mathbb C}^m$ into 
itself defined as follows: $(u\otimes v)x=u\langle x,v\rangle, x\in {\mathbb C}^m.$ 

Let
$$
{\mathbb H}_m:= \{A\in {\mathbb M}_m: A=A^{\ast}\}
$$   
be the set of all Hermitian matrices. Clearly, ${\mathbb H}_m$ is a linear space of dimension $m^2$ over the field 
of real numbers. For $A\in {\mathbb H}_m,$ the notation $A\succcurlyeq 0$ means that $A$
is positively semi-definite. A density matrix is a positively semi-definite Hermitian matrix 
of unit trace. The set of all $m\times m$ density matrices will be denoted by 
$$
{\mathcal S}_m:=\{S\in {\mathbb H}_m: S\succcurlyeq  0, {\rm tr}(S)=1\}.
$$ 
Density matrices are used in quantum mechanics to characterize the states 
of quantum systems. More generally, the states are represented by self-adjoint 
positively semidefinite operators of unit trace acting in an infinite-dimensional 
Hilbert space. In this case, density matrices of a large dimension $m$
could be used to approximate the states of the system.  

The goal of {\it quantum state tomography} is to estimate the density 
matrix for a system prepared in an unknown state based on specially 
designed measurements. Let $X\in {\mathbb H}_m$ be a Hermitian 
matrix ({\it an observable}) with spectral representation 
$X=\sum_{j=1}^{m'} \lambda_j P_j,$ where $m'\leq m,$
$\lambda_j\in {\mathbb R}, j=1,\dots, m'$ being the distinct eigenvalues of $X$  
and $P_j, j=1,\dots, m'$ being the corresponding eigenprojections. 
For a system prepared in state $\rho \in {\mathcal S}_m,$ possible 
outcomes of a measurement of observable $X$ are the eigenvalues 
$\lambda_j, j=1,\dots, m'$ and they occur with probabilities $p_j:= {\rm tr}(\rho P_j), j=1,\dots, m'.$
If $Y$ is a random variable representing such an outcome, then 
$$
{\mathbb E}_{\rho}Y= {\rm tr}(\rho X)=\langle \rho, X\rangle.
$$
In a simple model of quantum state tomography considered in this paper, an observable $X$ is sampled 
at random from some probability distribution $\Pi$ in ${\mathbb H}_m,$ ${\mathbb E}_{\rho}(Y|X)=\langle \rho, X\rangle$
and $Y=\langle \rho, X\rangle+\xi$ with noise $\xi$ such that ${\mathbb E}_{\rho}(\xi|X)=0.$  
Given a sample 
$X_1,\dots, X_n$ of $n$ i.i.d. copies of $X,$ $n$ measurements of observables $X_1,\dots, X_n$
are performed for a system identically prepared $n$ times in the same unknown state $\rho \in {\mathcal S}_m$
resulting in outcomes $Y_1,\dots, Y_n.$ 
This leads to the following {\it trace regression model} 
\begin{equation}
\label{trace_regression}
Y_j = \langle \rho, X_j \rangle + \xi_j, j=1,\dots, n
\end{equation}
with design variables $X_j, j=1,\dots, n,$ response variables $Y_j, j=1,\dots, n$ and noise $\xi_j, j=1,\dots, n$ satisfying the assumption ${\mathbb E}_{\rho}(\xi_j|X_j)=0, j=1,\dots, n$ and ${\mathbb E}_{\rho}(Y_j|X_j)=\langle \rho, X_j\rangle.$
The goal is to estimate the target 
density matrix $\rho$ based on the data $(X_1,Y_1), \dots, (X_n,Y_n),$
with the estimation error being measured by one of the statistically meaningful 
distances between density matrices such as the Schatten $p$-norm distances 
for $p\in [1,\infty]$ or quantum versions of Hellinger and Kullback-Leibler distances. 

This version of the problem of quantum state tomography has been intensively studied in the recent years. 
The noiseless case (quantum compressed sensing) was considered in \cite{gross2010quantum}
and \cite{gross2011recovering}. In these papers, sharp bounds on the number $n$ of measurements needed to recover 
a density matrix of rank $r$ were obtained based on a subtle argument (so called ``golfing scheme") utilizing matrix 
Bernstein type inequalities. These developments were related to an earlier work on low rank matrix completion 
\cite{candes2010power}. In the noisy case, trace regression problems have been studied by many authors 
(see, e.g., \cite{Koltchinskii2011oracle} and references therein). The main focus was on nuclear norm penalized 
least squares estimator (matrix LASSO) and related methods such as matrix Dantzig selector (see \cite{candes2011tight},
\cite{Koltchinskii2011nuclear}, \cite{negahban}, \cite{klopp2014noisy}).
In \cite{liu2011universal}, sharp bounds for matrix LASSO and matrix Dantzig selector, 
in particular, for Pauli measurements in quantum state tomography were obtained. 
Most of the results in these papers included upper bounds on the estimation error in Hilbert--Schmidt (Frobenius)
norm as well as low rank oracle inequalities (\cite{Koltchinskii2011nuclear}, \cite{Koltchinskii2011oracle}, \cite{Koltchinskii2013sharp}). In \cite{Koltchinskii2011nuclear}, an upper bound on the operator norm error
of a nuclear norm penalized modified least squares estimator was also proved. This result was further developed 
in \cite{Lounici_arxiv}.  In \cite{Koltchinskii2011neumann}, upper bounds and low rank oracle inequalities for von Neumann 
entropy penalized least squares estimators were studied (including the bounds on the error in Bures distance and quantum relative entropy distance). A rank penalized 
estimator of density matrix was studied in \cite{Alquier}. 
The minimax lower bounds on the Frobenius norm error for matrix completion problems 
in classes of matrices of rank $r$ were obtained in \cite{Koltchinskii2011nuclear} (the operator norm version could be found in \cite{Lounici_arxiv}).
In \cite{ma2013volume}, a method of deriving lower bounds for unitary invariant matrix norms (including Schatten $p$-norms) was developed and, among other matrix estimation problems, such bounds were obtained for matrix 
completion. Minimax lower bounds on the nuclear norm error in density matrix estimation were obtained in
\cite{flammia2012quantum}, where it was also shown that these bounds are attained (up to logarithmic factors)
for the matrix versions of LASSO and Dantzig selector. In our recent paper \cite{Koltch_Xia_15}, we derived
minimax lower bounds in classes of low rank density matrices for the whole range of Schatten $p$-norm distances 
as well as for Bures (quantum Hellinger) and quantum relative entropy distance. We also showed that these minimax bounds are attained (up to logarithmic factors) for von Neumann 
entropy penalized least squares estimators introduced in \cite{Koltchinskii2011neumann} 
simultaneously for Bures, relative entropy and Schatten $p$-norm distances for $p\in [1,2].$  

The current paper could be viewed as a continuation of \cite{Koltch_Xia_15}. Our main goal is to study a minimal distance estimator $\check \rho$ of $\rho$ (initially proposed in \cite{Koltchinskii2013remark}) 
defined as the projection of a simple unbiased estimator
$$
\hat Z = \frac{m^2}{n}\sum_{j=1}^n Y_j X_j 
$$
onto the convex set of density matrices ${\mathcal S}_m.$ 
We show that the minimax error rates established in \cite{Koltch_Xia_15}
for the classes of low rank density matrices are attained for this estimator up to logarithmic factors {\it in the whole range of Schatten $p$-norm 
distances} for $p\in [1,\infty]$ as well as for Bures and relative entropy distance. 
The proof of these results relies 
on simple properties of projections of Hermitian matrices onto the convex set ${\mathcal S}_m$ of density 
matrices (see theorems \ref{th:min-dist} and \ref{th_main_rank}) that might be of independent interest. 

Throughout the paper, $\langle \cdot, \cdot\rangle$ denotes either Hilbert--Schmidt inner product (defined above),
or (with a little abuse of notation) the canonical inner product of ${\mathbb C}^m.$ The corresponding norm in 
${\mathbb C}^m$ is denoted by $|\cdot|.$ For $A,B\geq 0,$ the notation $A\lesssim B$ means that $A\leq CB$
for a numerical constant $C>0,$ $A\gtrsim B$ means that $B\lesssim A$ and $A\asymp B$ means that $B\lesssim A\lesssim B.$ If needed, these signs might be provided with subscripts indicating that the constant is allowed to depend 
on parameters. Say, $A\lesssim_{\gamma}B$ would mean that $A\leq CB$ with $C$ depending on $\gamma.$

\section{Preliminaries}

\subsection{Distances between density matrices}

The Schatten $p$-norm of a matrix $A\in {\mathbb H}_m$ is defined as 
$$\|A\|_p := \biggl(\sum_{j=1}^m |\lambda_j(A)|^p\biggr)^{1/p}, p\in [1,+\infty],$$ 
where $\lambda_1(A)\geq \dots \geq \lambda_m(A)$ are the eigenvalues of $A$ arranged
in a non-increasing order. For $p=1,$ the norm $\|A\|_1$ is called the nuclear or the trace norm; for $p=2,$ $\|A\|_2$
is the Hilbert--Schmidt (generated by the Hilbert--Schmidt inner product) or Frobenius norm; for $p=+\infty,$ $\|A\|_{\infty}=\max_{1\leq j\leq m}|\lambda_j(A)|$ is called the operator or the spectral norm. Note that, for all $A\in {\mathbb H}_m,$ $[1,\infty]\ni p\mapsto \|A\|_p$
is a non-increasing function.
The following {\it interpolation inequality} is well known and can be easily deduced 
from a similar result for $\ell_p$-norms. Let $1\leq p<q<r\leq \infty$ and let $\mu\in [0,1]$ be such that 
$
\frac{\mu}{p}+ \frac{1-\mu}{r}= \frac{1}{q},
$ 
then 
\begin{equation}
\label{interpolation}
\|A\|_q \leq \|A\|_p^{\mu}\|A\|_r^{1-\mu},\ A\in {\mathbb H}_m. 
\end{equation}
In addition to the distances generated by the Schatten $p$-norms, the following two distances (extending well 
known distances between probability distributions used in the classical statistics) are of importance in quantum 
statistics: Bures distance and Kullback-Leibler divergence. {\it The Bures distance} is a quantum version of 
Hellinger distance and it is defined as follows:
$$
H^2(S_1,S_2):= 2-2{\rm tr}\sqrt{S_1^{1/2}S_2S_1^{1/2}}, S_1, S_2\in {\mathcal S}_m.
$$
The quantity ${\rm tr}\sqrt{S_1^{1/2}S_2S_1^{1/2}}$ is called {\it the fidelity} of states $S_1,S_2$
(a quantum version of Hellinger affinity). Note that $0\leq H^2(S_1,S_2)\leq 2$ and that $H(S_1,S_2)$
defines a metric in the space ${\mathcal S}_m.$ The non-commutative {\it Kullback-Leibler divergence}, 
or {\it relative entropy distance} is defined as
$$
K(S_1\|S_2):= {\rm tr}(S_1 \log S_1-S_1\log S_2), S_1,S_2\in {\mathcal S}_m.
$$
If $S_2$ is a density matrix of rank strictly smaller than $m,$ $\log S_2$ is not well defined 
and $K(S_1\|S_2):= +\infty.$ Clearly, $K(S_1\|S_2)$ is not a metric (in particular, it is not 
symmetric). It is well known that $K(S_1\|S_2)$ is the supremum 
of classical Kullback-Leibler divergences between the distributions of outcomes 
of all possible measurements  (represented by positive operator valued measures (POVM))
for the system prepared in states $S_1$ and $S_2.$ Similar property holds also for the Bures (Hellinger)
distance and for the nuclear norm distance $\|S_1-S_2\|_1$ which is the supremum of classical 
total variation distances between the distributions of outcomes of all measurements (see \cite{Nielsen2000}, \cite{Klauck2007}). 
These observations easily imply the following inequalities:
\begin{equation}
\label{compare_distances}
\frac{1}{4}\|S_1-S_2\|_1^2 \leq H^2(S_1,S_2)\leq K(S_1\|S_2)\wedge \|S_1-S_2\|_1
\end{equation}
(see also \cite{Koltchinskii2011neumann}).

\subsection{Sampling from an orthonormal basis}

Uniform sampling from an orthonormal basis is a model of design distribution 
in trace regression (\ref{trace_regression}) that has been frequently used in 
the literature on quantum compressed sensing (see, \cite{gross2010quantum}, \cite{gross2011recovering}). 
Let ${\mathcal E}:=\{E_1,\dots, E_{m^2}\}$ be an orthonormal basis of the space ${\mathbb H}_m$
of Hermitian matrices. 
Let 
$$
U:= \max_{1\leq j\leq m^2}\|E_j\|_{\infty}.
$$
Clearly, $U\leq 1$ and 
$$
1=\max_{1\leq j\leq m^2}\|E_j\|_{2}\leq m^{1/2} U,
$$
implying that $U\geq m^{-1/2}.$
In what follows, it will be assumed that $\Pi$ is a uniform distribution on the basis ${\mathcal E}.$
As a result, the response variables $Y_j, j=1,\dots, n$ of trace regression model (\ref{trace_regression}) could be viewed as noisy measurements of $n$ randomly picked Fourier 
coefficients of the target density matrix $\rho$ in basis ${\mathcal E}.$ This model includes, in particular,
the so called {\it Pauli measurements}, an important approach to quantum 
state tomography (see, e.g., \cite{gross2010quantum}, \cite{gross2011recovering}). 

{\bf Example: Pauli bases and Pauli measurements}. The space of observables for a single qubit system is the space 
${\mathbb H}_2$ of $2\times 2$ Hermitian matrices. Let 
\begin{equation*}
 \sigma_0:=\left(\begin{array}{cc}1&0\\0&1 \end{array}\right),\quad \sigma_1:=\left(\begin{array}{cc}0&1\\1&0 \end{array}\right),
  \quad \sigma_2:=\left(\begin{array}{cc}0&i\\-i&0 \end{array}\right),\quad \sigma_3:=\left(\begin{array}{cc}1&0\\0&-1 \end{array}\right).
 \end{equation*}
 The matrices $\sigma_1, \sigma_2, \sigma_3$ (often denoted $\sigma_x, \sigma_y, \sigma_z$)
 are called {\it Pauli matrices}. The matrices $W_i=\frac{1}{\sqrt{2}}\sigma_i,\ i=0,1,2, 3$
 form an orthonormal basis of the space ${\mathbb H}_2$ (the Pauli basis).  For a system 
 consisting of $k$ qubits, the space of observables is ${\mathbb H}_m,$ where $m=2^k.$
 The Pauli basis of this space is defined by tensorizing the Pauli basis of ${\mathbb H}_2:$
 it consists of $m^2=4^k$ tensor products $W_{i_1}\otimes\ldots\otimes W_{i_k}, (i_1,\ldots,i_k)
\in \left\{0,1,2,3\right\}^k.$ Let $E_1=W_0\otimes \ldots \otimes W_0$ and let $E_2,\dots, E_{m^2}$
be the rest of the matrices of the Pauli basis of ${\mathbb H}_m.$ It is straightforward to check that 
$E_1=\frac{1}{\sqrt{m}} I_m,$ where $I_m$ denotes $m\times m$ identity matrix (thus, $\frac{1}{\sqrt{m}}$
is the only eigenvalue of $E_1$). Matrices $E_2, \dots, E_{m^2}$ have eigenvalues $\pm\frac{1}{\sqrt{m}}.$
Therefore, $\|E_j\|_{\infty}=m^{-1/2},$ implying that, for the Pauli basis, $U=m^{-1/2}.$
The fact that the matrices of this basis have the smallest possible operator norms has been used 
in quantum compressed sensing
(see \cite{gross2010quantum}, \cite{gross2011recovering}, \cite{liu2011universal}).
Matrices $E_j$ have the following spectral representations: $E_j=\frac{1}{\sqrt{m}}P_j^+-\frac{1}{\sqrt{m}}P_j^-$
with eigenprojections $P_j^{+}, P_j^{-}, j=1,\dots, m^2$ (for $E_1,$ $P_1^{-}=0$). 
A measurement of $E_j$ for a $k$ qubit system prepared in state $\rho$ results in a random outcome $\tau_j$ with two possible values $\pm\frac{1}{\sqrt{m}}$
taken with probabilities $\big<\rho,P_j^{\pm}\big>.$ For random variable $\tau_j,$ 
${\mathbb E}_{\rho}\tau_j=\langle \rho,E_j\rangle.$ The density matrix $\rho$ admits the following representation in the Pauli 
basis:
$$
\rho=\sum_{j=1}^{m^2}\frac{\alpha_j}{\sqrt{m}}E_j 
$$
with $\alpha_1=1$ and with some $\alpha_j\in {\mathbb R}, j=2,\dots, m^2.$
This implies that ${\mathbb E}_{\rho}\tau_j=\frac{\alpha_j}{\sqrt{m}},$
$$\mathbb{P}_{\rho}\Bigl\{\tau_j=\pm\frac{1}{\sqrt{m}}\Bigr\}=\frac{1\pm \alpha_j}{2}$$
and $\text{Var}_{\rho}(\tau_j)=\frac{1-\alpha_j^2}{m}.$
Note that, for $j=1,$ $\alpha_1=1,$ $\mathbb{P}_{\rho}\Bigl\{\tau_1=\frac{1}{\sqrt{m}}\Bigr\}=1$
and  $\text{Var}_{\rho}(\tau_1)=0.$ For $j=2,\dots, m^2,$ $|\alpha_j|<1$ and 
$\text{Var}_{\rho}(\tau_j)>0.$ 

Let $\nu$ be picked at random from the 
set $\{1,\dots, m^2\}$ (with the uniform distribution) and let $X=E_{\nu}, Y=\tau_{\nu}$
(which corresponds to random sampling from the Pauli basis with a subsequent measurement
of observable $X$ resulting in the outcome $Y$). Then ${\mathbb E}_{\rho}(Y|X)=\langle \rho,X\rangle$
and ${\text Var}_{\rho}(Y|X)=\frac{1-\alpha_{\nu}^2}{m}.$ Moreover, we have 
$$
{\mathbb P}\Bigl\{{\text Var}_{\rho}(Y|X)\leq \frac{1}{2m}\Bigr\}=
{\mathbb P}\Bigl\{\alpha_{\nu}^2\geq \frac{1}{2}\Bigr\}\leq 
2{\mathbb E}\alpha_{\nu}^2 = \frac{2}{m} \sum_{j=1}^{m^2}\frac{\alpha_j^2}{m}
=\frac{2\|\rho\|_2^2}{m}.
$$
Since, for $\rho\in {\mathcal S}_m,$ $\|\rho\|_2\leq 1,$ this means that, for $m>2$ with probability 
at least $1-\frac{2}{m},$ ${\text Var}_{\rho}(Y|X)> \frac{1}{2m}.$ In other words, the number of $j=1,\dots, m^2$
such that ${\text Var}_{\rho}(\tau_j)> \frac{1}{2m}$ is at least $m^2 - 2m$ implying that, for the most of the values of $j,$
${\text Var}_{\rho}(\tau_j)\asymp \frac{1}{m}.$

The variance could be further reduced by repeating the measurement of the observable $X$
$K$ times (for a system identically prepared in state $\rho$) and averaging the outcomes 
of the resulting $K$ measurements. In this case, the response variable becomes $Y=\langle \rho, X\rangle +
\xi,$ where ${\mathbb E}_{\rho}(\xi|X)=0$ and ${\mathbb E}_{\rho}(\xi^2|X)=
\text{Var}_{\rho}(Y|X)=\frac{1-\alpha_{\nu}^2}{Km}.$ 
 
\subsection{Minimax lower bounds} 
 
In  \cite{Koltch_Xia_15}, the problem of density matrix estimation was studied in the case of trace regression 
model (\ref{trace_regression}) with i.i.d. random design variables $X_1,\dots, X_n$ sampled from the uniform 
distribution in an orthonormal basis ${\mathcal E}=\{E_1,\dots, E_{m^2}\}$ in two different settings:
trace regression with Gaussian noise and trace regression with a bounded response. 
In both cases, minimax lower bounds on the estimation error
of the unknown target density matrix $\rho$ of rank at most $r$  
were obtained for the Schatten $p$-norm 
distances ($p\in [1,+\infty]$) as well as for the Bures version of quantum Hellinger distance 
and for the quantum Kullback-Leibler (relative entropy) distance.  These results of \cite{Koltch_Xia_15} are stated 
below. 

Denote by ${\mathcal S}_{r,m}$ the set of all density matrices of rank at most $r$ ($1\leq r\leq m$).

\begin{assumption}[Trace regression with Gaussian noise]
\label{Gaussian_noise}
Let $(X,Y)$ be a random couple with $X$ being a random matrix sampled 
from the uniform distribution $\Pi$ in an orthonormal basis 
${\mathcal E}=\{E_1,\dots, E_{m^2}\}\subset {\mathbb H}_m.$
Suppose that, for some density matrix $\rho\in {\mathcal S}_m,$
$Y=\langle \rho,X\rangle +\xi,$ where $\xi$ is a mean zero normal random 
variable with variance $\sigma_{\xi}^2$ independent of $X.$
Let $(X_1,Y_1), \dots, (X_n,Y_n)$ be $n$ i.i.d. copies of $(X, Y).$ 
\end{assumption}

In this model, the level of the noise $\xi$ is characterized by its variance 
which should be involved in the error bound (this could be viewed as a normal 
approximation of the noise in the case when repeated measurements are performed 
for each observable $X_j$ with averaging of the outcomes).

\begin{theorem}
\label{minmaxthm1}
Suppose Assumption \ref{Gaussian_noise} holds. 
For all $p\in [1,+\infty],$ there exist constants $c,c'>0$ such that,  the following bounds hold:\footnote{Here ${\mathbb P}_{\rho}$
denotes a probability measure such that Assumption \ref{Gaussian_noise} is satisfied with density matrix $\rho.$}  
\begin{equation}
 \label{minmaxthm1boundq}
 \underset{\hat{\rho}}{\inf}\underset{\rho\in\mathcal{S}_{r,m}}{\sup}\mathbb{P}_{\rho}
 \biggl\{\|\hat{\rho}-\rho\|_p\geq c\biggl(r^{1/p}\frac{\sigma_{\xi}m^{\frac{3}{2}}}{\sqrt{n}}\bigwedge 
 \biggl(\frac{\sigma_{\xi}m^{3/2}}{\sqrt{n}}\biggr)^{1-\frac{1}{p}}\bigwedge 1\biggr)\biggr\}\geq c',
\end{equation}
 \begin{equation}
 \label{minmaxthm1boundH}
 \underset{\hat{\rho}}{\inf}\underset{\rho\in\mathcal{S}_{r,m}}{\sup}\mathbb{P}_{\rho}\biggl\{H^2(\hat{\rho},\rho)\geq c\biggl(r\frac{\sigma_{\xi}m^{\frac{3}{2}}}{\sqrt{n}}\bigwedge 1\biggr)\biggr\}\geq c',
 \end{equation}
and
 \begin{equation}
 \label{minmaxthm1boundK}
 \underset{\hat{\rho}}{\inf}\underset{\rho\in\mathcal{S}_{r,m}}{\sup}\mathbb{P}_{\rho}\biggl\{K(\rho\|\hat{\rho})\geq c\biggl(r\frac{\sigma_{\xi}m^{\frac{3}{2}}}{\sqrt{n}}\bigwedge 1\biggr)\biggr\}\geq c',
 \end{equation}
 where $\inf_{\hat{\rho}}$ denotes the infimum over all estimators $\hat{\rho}$ in $\mathcal{S}_m$
 based on the data $(X_1,Y_1), \dots, (X_n,Y_n)$ satisfying the Gaussian trace regression 
 model with noise variance $\sigma_{\xi}^2.$
\end{theorem}

The trace regression model with a bounded response is characterized by the size $U$ of the range of response variable 
$Y,$ which usually coincides with the bound on the operator norms of the basis matrices $E_j.$ 
It includes, in particular, Pauli measurements discussed above (for which $U=m^{-1/2}$).

\begin{assumption}[Trace regression with a bounded response]
\label{bounded_response}
Let $(X,Y)$ be a random couple with $X$ being a random matrix sampled 
from the uniform distribution $\Pi$ in an orthonormal basis ${\mathcal E}=\{E_1,\dots, E_{m^2}\}\subset {\mathbb H}_m$
with $U:=\max_{1\leq j\leq m^2}\|E_j\|_{\infty}$
and $Y$ being a random variable with values in the interval $[-U,U].$ Suppose that, for some density matrix $\rho\in {\mathcal S}_m,$
${\mathbb E}(Y|X)=\langle \rho, X\rangle$ a.s. 
Let $(X_1,Y_1), \dots, (X_n,Y_n)$ be $n$ i.i.d. copies of $(X, Y).$ 
\end{assumption}

Let ${\mathcal P}_{r,m}(U)$ denote the class of all distributions $P$ of $(X,Y)$ such that Assumption \ref{bounded_response} holds for some $U>0$ and ${\mathbb E}(Y|X)= \langle \rho_P, X\rangle$
for some $\rho_P\in {\mathcal S}_{r,m}.$ For a given $P\in {\mathcal P}_{r,m}(U),$ ${\mathbb P}_P$
denotes the corresponding probability measure such that $(X_1,Y_1),\dots, (X_n,Y_n)$ are 
i.i.d. copies of $(X,Y).$

\begin{theorem}
\label{minmaxthm3'''}
Suppose Assumption \ref{bounded_response} is satisfied and, 
for some constant $\gamma \in (0,1),$
\begin{equation}
\label{traceEj}
\Bigl|{\rm tr}(E_k)\Bigr|\leq (1-\gamma)U m,\ k=1,\dots, m^2.
\end{equation}
Then, for all $p\in [1,+\infty],$ there exist constants $c_{\gamma},c_{\gamma}'>0$ such that the following bounds hold:
\begin{equation}
 \label{minmaxthm1boundq_U'''}
 \underset{\hat{\rho}}{\inf}\underset{P\in\mathcal{P}_{r,m}(U)}{\sup}\mathbb{P}_{P}
 \biggl\{\|\hat{\rho}-\rho_P\|_p\geq c_{\gamma}\biggl(r^{1/p}\frac{U m^{\frac{3}{2}}}{\sqrt{n}}\bigwedge \biggl(\frac{U m^{3/2}}{\sqrt{n}}\biggr)^{1-\frac{1}{p}}\bigwedge 1\biggr)\biggr\}\geq c_{\gamma}',
\end{equation}
 \begin{equation}
 \label{minmaxthm1boundH_U'''}
 \underset{\hat{\rho}}{\inf}\underset{P\in\mathcal{P}_{r,m}(U)}{\sup}\mathbb{P}_{P}\biggl\{H^2(\hat{\rho},\rho_P)\geq c_{\gamma}\biggl(r\frac{U m^{\frac{3}{2}}}{\sqrt{n}}\bigwedge 1\biggr)\biggr\}\geq c_{\gamma}',
 \end{equation}
and
 \begin{equation}
 \label{minmaxthm1boundK_U'''}
 \underset{\hat{\rho}}{\inf}\underset{P\in\mathcal{P}_{r,m}(U)}{\sup}\mathbb{P}_{P}\biggl\{K(\rho_P\|\hat{\rho})\geq c_{\gamma}\biggl(r\frac{U m^{\frac{3}{2}}}{\sqrt{n}}\bigwedge 1\biggr)\biggr\}\geq c_{\gamma}',
 \end{equation}
 where $\inf_{\hat{\rho}}$ denotes the infimum over all estimators $\hat{\rho}$ in $\mathcal{S}_m$
 based on the i.i.d. data $(X_1,Y_1), \dots, (X_n,Y_n)$ sampled from $P.$ 
\end{theorem}

As it was pointed out in \cite{Koltch_Xia_15} (see Remark 12), if $\gamma$ in condition (\ref{traceEj}) is small 
enough (say, $\gamma<1-\frac{1}{\sqrt{2}}$), then, in a given orthonormal basis ${\mathcal E},$ there exists {\it at most one}  matrix $E_j$ such that ${\rm tr}(E_j)>(1-\gamma)Um.$ In the case of Pauli basis, such a matrix indeed 
exists and it is $E_1=W_0\otimes \dots \otimes W_0.$ Thus, Theorem \ref{minmaxthm3'''} does not apply directly to the Pauli measurement model. However,
the following result does hold (see \cite{Koltch_Xia_15}, Theorem 10).

\begin{theorem}
\label{minmaxthm3_Pauli}
Let $\{E_1,\dots, E_{m^2}\}$ be the Pauli basis in the space ${\mathbb H}_m$ of $m\times m$ Hermitian 
matrices and let $X_1,\dots, X_n$ be i.i.d. random variables sampled from the uniform distribution 
in $\{E_1,\dots, E_{m^2}\}.$ Let $Y_1,\dots, Y_n$ be outcomes of measurements of observables $X_1,\dots, X_n$
for the system being identically prepared $n$ times in state $\rho.$ The corresponding probability measure 
will be denoted by ${\mathbb P}_{\rho}.$
Then, for all $p\in [1,+\infty],$ there exist constants $c,c'>0$ such that the following bounds hold:
\begin{equation}
 \label{minmaxthm1boundq_U_P}
 \underset{\hat{\rho}}{\inf}\underset{\rho\in\mathcal{S}_{r,m}}{\sup}\mathbb{P}_{\rho}
 \biggl\{\|\hat{\rho}-\rho\|_p\geq c\biggl(r^{1/p}\frac{m}{\sqrt{n}}\bigwedge \biggl(\frac{m}{\sqrt{n}}\biggr)^{1-\frac{1}{p}}\bigwedge 1\biggr)\biggr\}\geq c',
\end{equation}
 \begin{equation}
 \label{minmaxthm1boundH_U_P}
 \underset{\hat{\rho}}{\inf}\underset{\rho\in\mathcal{S}_{r,m}}{\sup}\mathbb{P}_{\rho}\biggl\{H^2(\hat{\rho},\rho)\geq c\biggl(r\frac{m}{\sqrt{n}}\bigwedge 1\biggr)\biggr\}\geq c',
 \end{equation}
and
 \begin{equation}
 \label{minmaxthm1boundK_U_P}
 \underset{\hat{\rho}}{\inf}\underset{\rho\in\mathcal{S}_{r,m}}{\sup}\mathbb{P}_{\rho}\biggl\{K(\rho\|\hat{\rho})\geq c\biggl(r\frac{m}{\sqrt{n}}\bigwedge 1\biggr)\biggr\}\geq c',
 \end{equation}
 where $\inf_{\hat{\rho}}$ denotes the infimum over all estimators $\hat{\rho}$ in $\mathcal{S}_m$
 based on the data $(X_1,Y_1), \dots, (X_n,Y_n).$ 
\end{theorem}

It was also shown in \cite{Koltch_Xia_15} that, in the case of Schatten $p$-norm distances for $p\in [1,2],$
Bures distance and Kullback-Leibler distance, the minimax lower bounds of theorems \ref{minmaxthm1}, \ref{minmaxthm3'''}
and \ref{minmaxthm3_Pauli} are attained up to logarithmic factors in $m$ and $n$ for a penalized least squares 
estimator with von Neumann entropy penalty introduced in \cite{Koltchinskii2011neumann}. In the current paper, our main goal 
is to show that the minimax optimal rates are attained up to logarithmic factors for a very simple minimal distance 
estimator (that does not require any penalization) {\it in the whole range} of Schatten $p$-norms, $p\in [1,\infty],$
as well as for Bures and Kullback-Leibler distances.

\section{Main Results}
 
For the model of uniform sampling from an orthonormal basis ${\mathcal E}=\{E_1, \dots, E_{m^2}\},$
the following simple estimator of unknown state $\rho\in {\mathcal S}_m$ is unbiased:
$$
\hat Z:= \frac{m^2}{n}\sum_{j=1}^n Y_j X_j.
$$  
Indeed, 
$$
{\mathbb E}_{\rho}\hat Z= m^2 {\mathbb E}_{\rho}(YX)= m^2 {\mathbb E}({\mathbb E}_{\rho}(Y|X)X)=
m^2{\mathbb E}{\rm tr}(\rho X)X 
$$ 
$$ 
=m^2 {\mathbb E}\langle \rho, X\rangle X=m^2 \frac{1}{m^2}\sum_{j=1}^{m^2} \langle \rho, E_j\rangle E_j= \rho.
$$ 
Clearly, $\hat Z$ is not necessarily a density matrix. 

We will now define the minimal distance estimator $\check \rho$ as the projection of $\hat Z$ onto the convex set 
${\mathcal S}_m$ of all density matrices. More precisely,
for an arbitrary $Z\in {\mathbb H}_m,$ define 
\begin{equation}
\label{def_projection}
\pi_{{\mathcal S}_m}(Z):= {\rm argmin}_{S\in {\mathcal S}_m}\|Z-S\|_2^2.
\end{equation}
Clearly, $\pi_{{\mathcal S}_m}(Z)$ is the closest density matrix to $Z$ with respect 
to the Hilbert--Schmidt norm distance (that is, the projection 
of $Z$ onto ${\mathcal S}_m;$ such a closest density matrix exists in view of compactness of 
${\mathcal S}_m$ and it is unique in view of strict convexity of $S\mapsto \|Z-S\|_2^2$). 
Let 
$$
\check \rho := \pi_{{\mathcal S}_m}(\hat Z).
$$

\begin{remark}
This definition is equivalent to the following
\begin{eqnarray}
\label{mod_LS}
&
\nonumber
\check \rho :=
{\rm argmin}_{S\in {\mathcal S}_m}\biggl[-\frac{2}{n} \sum_{j=1}^n Y_j \langle S,X_j\rangle + m^{-2}\|S\|_2^2\biggr] 
=
\\
&
{\rm argmin}_{S\in {\mathcal S}_m}\biggl[\frac{1}{n}\sum_{j=1}^n Y_j^2 -\frac{2}{n} \sum_{j=1}^n Y_j \langle S,X_j\rangle + m^{-2}\|S\|_2^2\biggr] 
\end{eqnarray}
that was considered in  \cite{Koltchinskii2013remark} (in \cite{Koltchinskii2011nuclear}, similar estimators
involving nuclear norm penalty were studied). Note that replacing the term $m^{-2}\|S\|_2^2$ in the right hand side of 
(\ref{mod_LS}) by its unbiased estimator $n^{-1}\sum_{j=1}^n \langle S,X_j\rangle^2$ yields the usual least 
squares estimator 
\begin{equation}
\label{LS}
\hat \rho := {\rm argmin}_{S\in {\mathcal S}_m}\biggl[n^{-1}\sum_{j=1}^n (Y_j-\langle S,X_j\rangle)^2\biggr].
\end{equation}
Note that we also have 
\begin{equation}
\label{LS''}
\hat \rho := {\rm argmin}_{S\in {\mathcal S}_m}\biggl[n^{-1}\sum_{j=1}^n (Y_j-\langle S,X_j\rangle)^2+\eps \|S\|_1\biggr]
\end{equation}
since, for $S\in {\mathcal S}_m,$  $\|S\|_1={\rm tr}(S)=1.$ Thus, $\hat \rho$ coincides with the nuclear norm 
penalized least squares estimator (also called the matrix LASSO estimator) for any value of the regularization 
parameter $\eps.$
\end{remark}

We will show that the upper bounds on the error rates in Schatten $p$-norm distances for $p\in [1,\infty]$ and in Bures 
distance that match the minimax lower bounds of theorems \ref{minmaxthm1}, \ref{minmaxthm3'''}
and \ref{minmaxthm3_Pauli} up to logarithmic factors hold for the estimator $\check \rho.$ 
We will then introduce a simple modification of this estimator for which a matching upper 
bound holds also for Kullback-Leibler distance.  

First, we consider the case of Gaussian trace regression model (Assumption \ref{Gaussian_noise}). 
We need an additional assumption that $\sigma_{\xi}\geq \frac{U}{m^{1/2}}$ (the variance
of the noise is not too small).

\begin{theorem}
\label{minmaxthm1_upper}
Suppose Assumption \ref{Gaussian_noise} holds and $\sigma_{\xi}\geq \frac{U}{m^{1/2}}.$ 
For all $p\in [1,+\infty],$ there exists a constant $C>0$ such that, for all $A\geq 1$ the following bounds hold:  
\begin{equation}
 \label{minmaxthm1boundq_upper}
 \underset{\rho\in\mathcal{S}_{r,m}}{\sup}\mathbb{P}_{\rho}
 \biggl\{\|\check{\rho}-\rho\|_p\geq C\biggl(r^{1/p}\frac{\sigma_{\xi}m^{\frac{3}{2}}\sqrt{A\log (2m)}}{\sqrt{n}}\bigwedge 
 \biggl(\frac{\sigma_{\xi}m^{3/2}\sqrt{A\log (2m)}}{\sqrt{n}}\biggr)^{1-\frac{1}{p}}\bigwedge 1\biggr)\biggr\}\leq (2m)^{-A}
\end{equation}
and 
\begin{equation}
\label{minmaxthm1boundH_upper}
\underset{\rho\in\mathcal{S}_{r,m}}{\sup}\mathbb{P}_{\rho}\biggl\{H^2(\check{\rho},\rho)\geq c\biggl(r\frac{\sigma_{\xi}m^{\frac{3}{2}}\sqrt{A\log (2m)}}{\sqrt{n}}\bigwedge 1\biggr)\biggr\}\leq (2m)^{-A}.
\end{equation}
If $\sigma_{\xi}< \frac{U}{m^{1/2}},$ the bounds  
still hold with $\sigma_{\xi}$ replaced by $\frac{U}{m^{1/2}}.$
\end{theorem}
 
Similarly, in the case of trace regression with a bounded response, the following result holds.  
 
\begin{theorem}
\label{minmaxthm3'''_upper}
Suppose Assumption \ref{bounded_response} is satisfied. 
Then, for all $p\in [1,+\infty],$ there exists a constant $C>0$ such that, for all $A\geq 1,$  
the following bounds hold:
\begin{equation}
 \label{minmaxthm1boundq_U'''_upper}
 \underset{P\in\mathcal{P}_{r,m}(U)}{\sup}\mathbb{P}_{P}
 \biggl\{\|\check{\rho}-\rho_P\|_p\geq C\biggl(r^{1/p}\frac{U m^{\frac{3}{2}}\sqrt{A\log (2m)}}{\sqrt{n}}\bigwedge \biggl(\frac{U m^{3/2}\sqrt{A\log (2m)}}{\sqrt{n}}\biggr)^{1-\frac{1}{p}}\bigwedge 1\biggr)\biggr\}\leq (2m)^{-A}
\end{equation}
and 
\begin{equation}
\label{minmaxthm1boundH_U'''_upper}
\underset{P\in\mathcal{P}_{r,m}(U)}{\sup}\mathbb{P}_{P}\biggl\{H^2(\check{\rho},\rho_P)\geq C\biggl(r\frac{U m^{\frac{3}{2}}\sqrt{A\log (2m)}}{\sqrt{n}}\bigwedge 1\biggr)\biggr\}\leq (2m)^{-A}. 
\end{equation}
\end{theorem}
 
For completeness, we state also the upper bounds in the case of Pauli measurements 
(that immediately follow from Theorem \ref{minmaxthm3'''_upper}). 

\begin{theorem}
\label{minmaxthm3_Pauli_upper}
Suppose the assumptions of Theorem \ref{minmaxthm3_Pauli} hold. 
Then, for all $p\in [1,+\infty],$ there exists a constant $C$ such that, for all $A\geq 1,$ the following bounds hold:
\begin{equation}
 \label{minmaxthm1boundq_U_P}
 \underset{\rho\in\mathcal{S}_{r,m}}{\sup}\mathbb{P}_{\rho}
 \biggl\{\|\check{\rho}-\rho\|_p\geq c\biggl(r^{1/p}\frac{m\sqrt{A\log (2m)}}{\sqrt{n}}\bigwedge \biggl(\frac{m\sqrt{A\log (2m)}}{\sqrt{n}}\biggr)^{1-\frac{1}{p}}\bigwedge 1\biggr)\biggr\}\leq (2m)^{-A}
\end{equation}
and
 \begin{equation}
 \label{minmaxthm1boundH_U_P}
 \underset{\rho\in\mathcal{S}_{r,m}}{\sup}\mathbb{P}_{\rho}\biggl\{H^2(\check{\rho},\rho)\geq c\biggl(r\frac{m}{\sqrt{n}}\bigwedge 1\biggr)\biggr\}\leq (2m)^{-A}. 
 \end{equation}
\end{theorem}

The proof of these results relies on the following fact that might be of independent interest and that essentially 
shows that $\pi_{{\mathcal S}_m}(Z)$ is the closest density matrix to $Z$ not only in the Hilbert--Schmidt
norm distance, but also in the operator norm distance.   

\begin{theorem}
\label{th:min-dist}
For all $Z\in {\mathbb H}_m,$
$$
\|Z-\pi_{{\mathcal S}_m}(Z)\|_{\infty}= \inf_{S\in {\mathcal S}_m}\|Z-S\|_{\infty}.
$$
\end{theorem}

The proof of this theorem will be given in Section \ref{last}. Here we use it to establish the next result 
that is the main ingredient of the proofs of  theorems \ref{minmaxthm1_upper}, \ref{minmaxthm3'''_upper} and \ref{minmaxthm3_Pauli_upper}.

\begin{theorem}
\label{th_main_rank}
Let $p\in [1,+\infty].$
For all $Z\in {\mathbb H}_m$ and all $S\in {\mathcal S}_{r,m},$ 
$$
\|\pi_{{\mathcal S}_m}(Z)-S\|_p\leq \min\Bigl(2^{3/p+1}r^{1/p}\|Z-S\|_{\infty}, 2\|Z-S\|_{\infty}^{1-1/p}\Bigr).
$$
\end{theorem}

The proof relies on Theorem \ref{th:min-dist} and on a simple lemma stated below.

\begin{lemma}
\label{rank_r}
Let $S,S'\in {\mathcal S}_m$ and ${\rm rank}(S)=r.$ 
Then, for all $p\in [1,\infty],$ 
$$
\|S'-S\|_p \leq \min\Bigl((8r)^{1/p} \|S'-S\|_{\infty}, 2^{1/p}\|S'-S\|_{\infty}^{1-1/p}\Bigr).
$$
\end{lemma}

\begin{proof}
Let $S= \sum_{j=1}^r \lambda_j (\phi_j\otimes \phi_j)$ be the spectral 
decomposition of $S$ with eigenvalues $\lambda_j$ and eigenvectors $\phi_j.$ 
Let $L:={\rm supp}(S)$ be the linear span of vectors $\phi_1,\dots, \phi_r\in {\mathbb C}^m.$
Denote by $P_L, P_{L^{\perp}}$ the orthogonal projection operators onto subspace $L$ and 
its orthogonal complement $L^{\perp},$ respectively. 
We will need the following projection operators ${\mathcal P}_L, {\mathcal P}_L^{\perp}:{\mathbb H}_m\mapsto {\mathbb H}_m:$
$$
{\mathcal P}_L^{\perp} (A)= P_{L^{\perp}}AP_{L^{\perp}},\ \ {\mathcal P}_L(A)=A-P_{L^{\perp}}AP_{L^{\perp}},\ \ 
A\in {\mathbb H}_m.
$$
The following bounds are obvious:
$$
\|S\|_1= 1= \|S'\|_1 = \|S'-S+S\|_1 = 
\|{\mathcal P}_L(S'-S)+{\mathcal P}_L^{\perp}(S'-S)+S\|_1
$$ 
$$
\geq \|{\mathcal P}_L^{\perp}(S'-S)+S\|_1-\|{\mathcal P}_L(S'-S)\|_1.
$$
Since $S=P_LSP_L,$ we can use the pinching inequality for unitary invariant norm $\|\cdot\|_1$
(see \cite{Bhatia}, p. 97) to get:
$$
\|{\mathcal P}_L^{\perp}(S'-S)+S\|_1= \|P_L S P_L + P_{L^{\perp}}(S'-S)P_{L^{\perp}}\|_1
$$
$$
= \|P_L S P_L \|_1+  \|P_{L^{\perp}}(S'-S)P_{L^{\perp}}\|_1= \|S\|_1+ \|{\mathcal P}_L^{\perp}(S'-S)\|_1.
$$ 
Therefore, 
$$
\|S\|_1 \geq \|S\|_1+ \|{\mathcal P}_L^{\perp}(S'-S)\|_1- \|{\mathcal P}_L(S'-S)\|_1,
$$
implying that 
$$
\|{\mathcal P}_L^{\perp}(S'-S)\|_1\leq \|{\mathcal P}_L(S'-S)\|_1.
$$
It follows from the last bound that 
$$
\|S'-S\|_1 =\|{\mathcal P}_L(S'-S)+{\mathcal P}_L^{\perp}(S'-S)\|_1
\leq 2 \|{\mathcal P}_L(S'-S)\|_1.
$$
Since ${\rm dim}(L)=r,$ the matrix ${\mathcal P}_L(S'-S)$ is of rank at most $2r.$
This implies that 
$$
\|{\mathcal P}_L(S'-S)\|_1\leq 2r \|{\mathcal P}_L(S'-S)\|_{\infty}
$$
$$
\leq 2r (\|(S'-S)P_L\|_{\infty}+ \|P_L(S'-S)P_{L^{\perp}}\|_{\infty})
\leq 4r\|S'-S\|_{\infty}. 
$$
Therefore, 
$
\|S'-S\|_1\leq 8r \|S'-S\|_{\infty},
$
and since also  $\|S'-S\|_1\leq 2, S,S'\in {\mathcal S}_m,$ we conclude that 
$$
\|S'-S\|_1\leq \min(8r \|S'-S\|_{\infty},2).
$$
Together with interpolation inequality this yields that for all $p\in [1,\infty]$
$$
\|S'-S\|_p \leq \|S'-S\|_1^{1/p}\|S'-S\|_{\infty}^{1-1/p}
\leq \min\Bigl((8r)^{1/p} \|S'-S\|_{\infty}, 2^{1/p}\|S'-S\|_{\infty}^{1-1/p}\Bigr).
$$

\end{proof}

\begin{proof} We now prove Theorem \ref{th_main_rank}. 
It immediately follows from Theorem \ref{th:min-dist} that, for all $S\in {\mathcal S}_m,$ 
$$
\|\pi_{{\mathcal S}_m}(Z)-S\|_{\infty}\leq \|\pi_{{\mathcal S}_m}(Z)-Z\|_{\infty}+ \|Z-S\|_{\infty}
\leq 2\|Z-S\|_{\infty}.
$$
If $S\in {\mathcal S}_m$ is a density matrix of rank $r,$ the last bound could be combined with 
the bound of Lemma \ref{rank_r} to get that for all $p\in [1,+\infty]$
$$
\|\pi_{{\mathcal S}_m}(Z)-S\|_p\leq \min\Bigl(2^{3/p+1}r^{1/p}\|Z-S\|_{\infty}, 2\|Z-S\|_{\infty}^{1-1/p}\Bigr).
$$

\end{proof}

\begin{proof}
We now turn to the proof of theorems \ref{minmaxthm1_upper}, \ref{minmaxthm3'''_upper} and \ref{minmaxthm3_Pauli_upper}. To this end, we use the bound of Theorem \ref{th_main_rank}
with $Z=\hat Z$ and $S=\rho\in {\mathcal S}_{r,m}$ that yields: 
\begin{equation}
\label{bd_od}
\|\check \rho-\rho\|_p\leq \min\Bigl(2^{3/p+1}r^{1/p}\|\hat Z-\rho\|_{\infty}, 2\|\hat Z-\rho\|_{\infty}^{1-1/p}\Bigr).
\end{equation}
The control of 
$$
\|\hat Z-\rho\|_{\infty}= \biggl\| \frac{m^2}{n}\sum_{j=1}^n Y_j X_j-\rho\biggr\|_{\infty}
$$
is based on a standard application of matrix Bernstein type inequalities. We give a detailed argument 
for completeness. 
Note that $\|\check \rho-\rho\|_p$ in the left-hand side 
of bound (\ref{bd_od}) is upper bounded by $2,$ so, if Bernstein bound on $\|\hat Z-\rho\|_{\infty}$
is larger than $1$ (or even $\gtrsim 1$), it could be replaced by the trivial bound equal to $1.$ 
In the case of Theorem \ref{minmaxthm3'''_upper},
we use the following version of
Bernstein inequality for i.i.d. bounded random matrices (see, e.g., \cite{tropp2012user}).

\begin{lemma}
\label{Bernstein_standard}
Let $V, V_1,\dots, V_n$ be i.i.d. random matrices in ${\mathbb H}_m$ with ${\mathbb E}V=0.$
Suppose that, for some constant $U>0,$ $\|V\|_{\infty}\leq U$ a.s. Let $\sigma^2:= \|{\mathbb E}V^2\|_{\infty}.$ 
Then, for all $t>0$ with probability at least $1-e^{-t},$
$$
\biggl\|\frac{V_1+\dots +V_n}{n}\biggr\|_{\infty}\leq 2\biggl[\sigma \sqrt{\frac{t+\log(2m)}{n}} \bigvee U\frac{t+\log (2m)}{n}\biggr].
$$
\end{lemma}

For $V= YX - {\mathbb E}(YX),$ we get, under Assumption \ref{bounded_response}, 
that 
$$
\sigma^2=\|{\mathbb E}V^2\|_{\infty} \leq \|{\mathbb E}(Y^2 X^2)\|_{\infty} \leq U^2 \|{\mathbb E}X^2\|_{\infty}.
$$
It is also well known that, under the same assumption, $\|{\mathbb E}X^2\|_{\infty}= m^{-1}.$
[Indeed, if $\{e_j, j=1,\dots, m\}$ is an orthonormal basis of ${\mathbb C}^m,$
then 
$$
\|{\mathbb E}X^2\|_{\infty} = \sup_{v\in {\mathbb C}^m,|v|\leq 1}{\mathbb E}\langle X^2 v,v\rangle
=  \sup_{v\in {\mathbb C}^m,|v|\leq 1}{\mathbb E}|Xv|^2= 
 \sup_{v\in {\mathbb C}^m,|v|\leq 1}{\mathbb E}\sum_{j=1}^m|\langle Xv, e_j\rangle|^2
$$
$$
=\sup_{v\in {\mathbb C}^m,|v|\leq 1}{\mathbb E}\sum_{j=1}^m|\langle X, v\otimes e_j\rangle|^2
=
\sup_{v\in {\mathbb C}^m,|v|\leq 1}\sum_{j=1}^m
m^{-2}\sum_{k=1}^{m^2} |\langle E_k, v\otimes e_j\rangle|^2
=
\sup_{v\in {\mathbb C}^m,|v|\leq 1}m^{-2}\sum_{j=1}^m
\|v\otimes e_j\|_2^2= 
$$
$$
\sup_{v\in {\mathbb C}^m,|v|\leq 1}m^{-2}\sum_{j=1}^m
|v|^2 |e_j|^2=  m^{-1}].
$$
We use the bound of Lemma \ref{Bernstein_standard} with $t=A\log (2m), A\geq 1$ to get that  
with probability at least $1-(2m)^{-A},$
$$
\biggl\| \frac{m^2}{n}\sum_{j=1}^n Y_j X_j-\rho\biggr\|_{\infty}
\leq C\biggl[Um^{3/2}\sqrt{\frac{A\log(2m)}{n}}\bigvee \frac{U^2 m^2 A\log(2m)}{n}\biggr]
$$
with some absolute constant $C\geq 1.$ If 
$$
\frac{U^2 m^2 A\log(2m)}{n}\geq Um^{3/2}\sqrt{\frac{A\log(2m)}{n}},
$$
then $Um^{1/2}\sqrt{\frac{A\log (2m)}{n}}\geq 1$ implying that $Um^{3/2}\sqrt{\frac{A\log (2m)}{n}}\geq 1.$
Thus, when the bound on $\|\hat Z-\rho\|_{\infty}$ is substituted in bound (\ref{bd_od}), it is enough to keep 
only the first term $Um^{3/2}\sqrt{\frac{A\log(2m)}{n}},$ the second term could be dropped. This implies 
that with some constant $C'>0$ (that does not depend on $\rho\in {\mathcal S}_{r,m}$) the inequality 
$$ 
\|\check \rho-\rho\|_p\leq C' \biggl(r^{1/p}\frac{U m^{\frac{3}{2}}\sqrt{A\log (2m)}}{\sqrt{n}}\bigwedge \biggl(\frac{U m^{3/2}\sqrt{A\log (2m)}}{\sqrt{n}}\biggr)^{1-\frac{1}{p}}\bigwedge 1\biggr)
$$
holds with probability at least $1-(2m)^{-A},$ implying the first bound of Theorem \ref{minmaxthm3'''_upper}.
The second bound immediately follows from the inequality $H^2(\check \rho, \rho)\leq \|\check \rho-\rho\|_1$
(see (\ref{compare_distances})). Theorem \ref{minmaxthm3_Pauli_upper} is an immediate consequence 
of Theorem \ref{minmaxthm3'''_upper}.

The proof of Theorem \ref{minmaxthm1_upper} is very similar. In this case, Assumption \ref{Gaussian_noise}
holds and it is natural to split $\hat Z-\rho$ into two parts
\begin{equation}
\label{two_parts}
\hat Z-\rho = \frac{m^2}{n}\sum_{j=1}^n \langle \rho, X_j\rangle X_j -\rho +\frac{m^2}{n}\sum_{j=1}^n \xi_j X_j.
\end{equation}
and to bound $\|\hat Z-\rho\|_{\infty}$ by triangle inequality.
For the first part, an application of matrix Bernstein inequality of Lemma \ref{Bernstein_standard} yields the bound 
\begin{equation}
\label{part_1}
\biggl\|\frac{m^2}{n}\sum_{j=1}^n \langle \rho, X_j\rangle X_j -\rho\biggr\|_{\infty}
\leq C\biggl[Um \sqrt{\frac{A\log (2m)}{n}}\bigvee \frac{U^2 m^2 A\log (2m)}{n}\biggr]
\end{equation}
that holds for some absolute constant $C\geq 1$ with probability at least $1-(2m)^{-A}.$
Indeed, in this case $V=\langle \rho, X\rangle X-{\mathbb E}\langle \rho, X\rangle X$ and 
$$
\sigma^2 \leq  \|{\mathbb E}\langle \rho,X\rangle^2 X^2\|_{\infty}
\leq U^2 {\mathbb E}\langle \rho,X\rangle^2= \frac{U^2\|\rho\|_2^2}{m^2} \leq 
\frac{U^2}{m^2},
$$
$$
\|\langle \rho,X\rangle X\|_{\infty}\leq \|\rho\|_1 \|X\|_{\infty}^2 \leq \|X\|_{\infty}^2 \leq U^2,
$$
and Lemma \ref{Bernstein_standard} implies (\ref{part_1}).  
As before, if $\frac{U^2 m^2 A\log (2m)}{n}\geq Um \sqrt{\frac{A\log (2m)}{n}},$ then 
$Um \sqrt{\frac{A\log (2m)}{n}}\geq 1.$ Thus, the second term  $\frac{U^2 m^2 A\log (2m)}{n}$
could be dropped when the bound on $\|\hat Z-\rho\|_{\infty}$  (for which the right hand side 
of (\ref{part_1}) is a part) is substituted in (\ref{bd_od}).  

As to the second part of representation 
(\ref{two_parts}) that involves normal random variables $\xi_j,$
it is bounded using another version of matrix Bernstein inequality 
for not necessarily bounded random matrices (see \cite{Koltchinskii2011neumann}, \cite{Koltchinskii2011oracle}, \cite{Koltchinskii2013sharp}).

\begin{lemma}
\label{Bernstein_unbounded}
Let $V, V_1,\dots, V_n$ be i.i.d. random matrices in ${\mathbb H}_m$ with ${\mathbb E}V=0.$
Suppose that, for some $\alpha\geq 1,$ $U^{(\alpha)}:=2\bigl\|\|V\|_{\infty}\bigr\|_{\psi_{\alpha}}<+\infty.$
\footnote{Here $\|\cdot\|_{\psi_{\alpha}}$ denotes the $\psi_{\alpha}$ Orlicz norm in the space of random variables 
defined as follows:
$$
\|\eta\|_{\psi_{\alpha}}:= \inf\biggl\{c>0: {\mathbb E}\exp\Bigl\{\frac{|\eta|^{\alpha}}{c^{\alpha}}\Bigr\}\leq 2\biggr\}.
$$
}
Let $\sigma^2:= \|{\mathbb E}V^2\|_{\infty}.$ 
Then, for all $t>0$ with probability at least $1-e^{-t},$
$$
\biggl\|\frac{V_1+\dots +V_n}{n}\biggr\|_{\infty}\leq C\biggl[\sigma \sqrt{\frac{t+\log(2m)}{n}} \bigvee U^{(\alpha)}
\log^{1/\alpha} \biggl(\frac{U^{(\alpha)}}{\sigma}\biggr)\frac{t+\log (2m)}{n}\biggr].
$$
\end{lemma}

We apply the bound of Lemma \ref{Bernstein_unbounded} in the case when $V:= \xi X, \alpha=2$
and $t=A\log (2m)$ for $A\geq 1.$
By an easy computation, 
$$
\sigma^2 = \sigma_{\xi}^2 \|{\mathbb E}X^2\|_{\infty}=\frac{\sigma_{\xi}^2}{m}
$$
and 
$$
U^{(2)} = 2\bigl\|\xi \|X\|_{\infty}\bigr\|_{\psi_2}\leq 2U\|\xi\|_{\psi_2}\leq 4\sigma_{\xi}U. 
$$
This yields the following bound
\begin{equation}
\label{part_2}
\biggl\|\frac{m^2}{n}\sum_{j=1}^n \xi_j X_j\biggr\|_{\infty}
\leq C \biggl[\sigma_{\xi}m^{3/2}\sqrt{\frac{A\log (2m)}{n}}\bigvee \sigma_{\xi}U\frac{m^2 A \log(2m) \log^{1/2}(4U\sqrt{m})}{n}\biggr]
\end{equation}
that holds with probability at least $1-(2m)^{-A}$ and with some absolute constant $C\geq 1.$
If the second term in the maximum in the right hand side of (\ref{part_2}) is dominant, 
then $Um^{1/2}\sqrt{\frac{A\log (2m)}{n}}\log^{1/2}(4U\sqrt{m})\geq 1.$ Under the condition that 
$\sigma_{\xi}\geq Um^{-1/2},$ this implies that also $\sigma_{\xi}m^{3/2}\sqrt{\frac{A\log (2m)}{n}}\gtrsim 1.$
Thus, when the bound in the right hand side of (\ref{part_2}) (used to control $\|\hat Z-\rho\|_{\infty}$) is substituted 
in (\ref{bd_od}), it is enough to keep only the first term in the maximum. Finally, under the assumption 
$\sigma_{\xi}\geq Um^{-1/2},$ the first term of bound (\ref{part_2}) dominates the first term of (\ref{part_1}),
so, only this term is needed to control $\|\hat Z-\rho\|_{\infty}$ in bound (\ref{bd_od}). These considerations 
imply the bound 
$$ 
\|\check \rho-\rho\|_p\leq C' \biggl(r^{1/p}\frac{\sigma_{\xi} m^{\frac{3}{2}}\sqrt{A\log (2m)}}{\sqrt{n}}\bigwedge 
\biggl(\frac{\sigma_{\xi} m^{3/2}\sqrt{A\log (2m)}}{\sqrt{n}}\biggr)^{1-\frac{1}{p}}\bigwedge 1\biggr)
$$
that holds with some constant $C'>0$ (that does not depend on $\rho\in {\mathcal S}_{r,m}$) and with probability at least $1-(2m)^{-A}.$ The first bound of Theorem \ref{minmaxthm1_upper} now follows for all $p\in [1,\infty]$
(which also implies the second bound in view of (\ref{compare_distances})). 

\end{proof}

It turns out that for a slightly modified version of estimator $\check \rho,$ minimax lower bounds are also 
attained (up to logarithmic factors)  in the case of Kullback-Leibler distance.  For $S\in {\mathcal S}_m$
and $\delta\in [0,1],$ define $S_{\delta}= (1-\delta)S+\delta \frac{I_m}{m}.$ Clearly, $S_{\delta}\in {\mathcal S}_m.$
Let 
$
{\mathcal S}_{m,\delta}:=\{S_{\delta}:S\in {\mathcal S}_m\}. 
$
Define $\pi_{{\mathcal S}_{m,\delta}}(Z)$ the projection of $Z\in {\mathbb H}_m$
onto the convex set ${\mathcal S}_{m,\delta}:$
$$
\pi_{{\mathcal S}_{m,\delta}}(Z):= {\rm argmin}_{S\in {{\mathcal S}_{m,\delta}}} \|Z-S\|_2^2.
$$
Let 
$$
\check \rho_{\delta}:= \pi_{{\mathcal S}_{m,\delta}}(\hat Z)
$$
with $\check \rho_0=\check \rho.$
We will prove the following versions of theorems \ref{minmaxthm1_upper}, \ref{minmaxthm3'''_upper}
and \ref{minmaxthm3_Pauli_upper} for the estimator $\check \rho_{\delta}.$

\begin{theorem}
\label{minmaxthm1_upper''}
Suppose Assumption \ref{Gaussian_noise} holds, $\sigma_{\xi}\geq \frac{U}{m^{1/2}}$
and 
$$
\delta \leq \frac{\sigma_{\xi}m^{\frac{3}{2}}\sqrt{\log (2m)}}{\sqrt{n}}
\bigwedge 1.
$$
Then bounds (\ref{minmaxthm1boundq_upper}) and (\ref{minmaxthm1boundH_upper}) hold for
estimator $\check \rho_{\delta}.$ 
Moreover, for $A\geq 1,$ define 
$$
\lambda := \frac{r\sigma_{\xi}m^{5/2}\sqrt{\frac{A\log (2m)}{n}}\bigwedge m}{\delta}.
$$
Then, for some constant $c>0,$
\begin{equation}
\label{minmaxthm1boundKL_upper}
\underset{\rho\in\mathcal{S}_{r,m}}{\sup}\mathbb{P}_{\rho}\biggl\{K(\rho\| \check \rho_{\delta})
\geq c\biggl(r\frac{\sigma_{\xi}m^{\frac{3}{2}}\sqrt{A\log (2m)}}{\sqrt{n}}\bigwedge 1\biggr)
\log (1+c\lambda)\biggr\}\leq (2m)^{-A}.
\end{equation}
If $\sigma_{\xi}< \frac{U}{m^{1/2}},$ the bounds  
still hold with $\sigma_{\xi}$ replaced by $\frac{U}{m^{1/2}}.$
\end{theorem}

\begin{theorem}
\label{minmaxthm3'''_upper''}
Suppose Assumption \ref{bounded_response} is satisfied 
and 
$$
\delta \leq \frac{U m^{\frac{3}{2}}\sqrt{\log (2m)}}{\sqrt{n}}
\bigwedge 1.
$$
Then (\ref{minmaxthm1boundq_U'''_upper}) and (\ref{minmaxthm1boundH_U'''_upper}) hold for
estimator $\check \rho_{\delta}.$
Moreover, for $A\geq 1,$
define 
$$
\lambda := \frac{r U m^{5/2}\sqrt{\frac{A\log (2m)}{n}}\bigwedge m}{\delta}.
$$
Then, for some constant $c>0,$
\begin{equation}
\label{minmaxthm1boundKL_upper''}
\underset{\rho\in\mathcal{S}_{r,m}}{\sup}\mathbb{P}_{\rho}\biggl\{K(\rho\| \check \rho_{\delta})
\geq c\biggl(r\frac{U m^{\frac{3}{2}}\sqrt{A\log (2m)}}{\sqrt{n}}\bigwedge 1\biggr)
\log (1+ c\lambda)\biggr\}\leq (2m)^{-A}.
\end{equation}
\end{theorem}

\medskip

\begin{theorem}
\label{minmaxthm3_Pauli_upper''}
Suppose the assumptions of Theorem \ref{minmaxthm3_Pauli} hold
and 
$$
\delta \leq \frac{m\sqrt{\log (2m)}}{\sqrt{n}}
\bigwedge 1.
$$
Then (\ref{minmaxthm1boundq_U_P}) and (\ref{minmaxthm1boundH_U_P}) hold for
estimator $\check \rho_{\delta}.$
Moreover, for $A\geq 1,$ define 
$$
\lambda := \frac{r m^2 \sqrt{\frac{A\log (2m)}{n}}\bigwedge m}{\delta}.
$$
Then, for some constant $c>0,$
\begin{equation}
\label{minmaxthm1boundKL_Pauli_upper''}
\underset{\rho\in\mathcal{S}_{r,m}}{\sup}\mathbb{P}_{\rho}\biggl\{K(\rho\| \check \rho_{\delta})
\geq c\biggl(r\frac{m\sqrt{A\log (2m)}}{\sqrt{n}}\bigwedge 1\biggr)
\log (1+c\lambda)\biggr\}\leq (2m)^{-A}.
\end{equation}
\end{theorem}

\medskip
 
\begin{remark}
If, under the assumptions of Theorem \ref{minmaxthm3'''_upper''}, we choose 
$$
\delta =\frac{U m^{\frac{3}{2}}\sqrt{\log (2m)}}{\sqrt{n}}
\bigwedge 1,
$$
then the logarithmic factor in bound (\ref{minmaxthm1boundKL_upper''}) satisfies the inequality
$$
\log (1+c\lambda)\leq \log (1+crm\sqrt{A}),
$$
so it is of the order $\log m.$
Under the assumptions of Theorem \ref{minmaxthm1_upper''}, this would require the choice of $\delta$
$$
\delta = \frac{\sigma_{\xi}m^{\frac{3}{2}}\sqrt{\log (2m)}}{\sqrt{n}}
\bigwedge 1,
$$
so $\delta $ would depend on an unknown parameter $\sigma_{\xi}.$ 
Replacing $\sigma_{\xi}$ in the definition of $\delta$ by the lower bound $Um^{-1/2}$
would result in a logarithmic factor $\lesssim \log \biggl(1+crm\sqrt{A}\frac{\sigma_{\xi}}{Um^{-1/2}}\biggr).$ 
\end{remark}

\medskip

\begin{proof} We start with the following modification of Theorem \ref{th_main_rank}.

\begin{lemma}
\label{mix}
Let $p\in [1,\infty].$ For all $Z\in {\mathbb H}_m$ and all $S\in {\mathcal S}_{r,m},$ the following bound 
holds:
$$
\|\pi_{{\mathcal S}_{m,\delta}}(Z)-S\|_p \leq 
\min \biggl(2^{3/p+1}r^{1/p}\Bigl(\|Z-S\|_{\infty}+2\delta\Bigr), 2(1-\delta)^{1/p}\Bigl(\|Z-S\|_{\infty}+2\delta\Bigr)^{1-1/p}\biggr)+2\delta.
$$
\end{lemma}

\begin{proof}
The following formula is straightforward: for $\delta\in [0,1),$
$$
\pi_{{\mathcal S}_{m,\delta}}(Z)= (1-\delta)\pi_{{\mathcal S}_m}\biggl(\frac{Z}{1-\delta}-\frac{\delta}{1-\delta}\frac{I_m}{m}\biggr)+\delta \frac{I_m}{m}.
$$
Indeed, $\pi_{{\mathcal S}_{m,\delta}}(Z)$ coincides with $(1-\delta)S^{'}+\delta \frac{I_m}{m},$
where 
$$
S' := {\rm argmin}_{S\in {\mathcal S}_m}\biggl\|Z-(1-\delta)S-\delta \frac{I_m}{m}\biggr\|_2^2
$$
$$
= {\rm argmin}_{S\in {\mathcal S}_m}\biggl\|\frac{Z}{1-\delta}- \frac{\delta}{1-\delta}\frac{I_m}{m}-S\biggr\|_2^2
=\pi_{{\mathcal S}_m}\biggl(\frac{Z}{1-\delta}-\frac{\delta}{1-\delta}\frac{I_m}{m}\biggr),
$$ 
implying the claim.

Let $S\in {\mathcal S}_{r,m}.$ Then, for $p\in [1,\infty],$
\begin{eqnarray}
\label{odin-odin}
&
\|\pi_{{\mathcal S}_{m,\delta}}(Z)-S\|_p 
\leq 
\|\pi_{{\mathcal S}_{m,\delta}}(Z)-S_{\delta}\|_p+\|S_{\delta}-S\|_p 
\\
&
\nonumber
\leq 
(1-\delta)\biggl\|\pi_{{\mathcal S}_m}\biggl(\frac{Z}{1-\delta}-\frac{\delta}{1-\delta}\frac{I_m}{m}\biggr)-S\biggr\|_p
+ 2\delta. 
\end{eqnarray}
To control the first term in the right hand side, we use the bound of Theorem \ref{th_main_rank},
which requires bounding $\Bigl\|\frac{Z}{1-\delta}-\frac{\delta}{1-\delta}\frac{I_m}{m}-S\Bigr\|_{\infty}.$
We have 
\begin{eqnarray}
\label{dva-dva}
&
\biggl\|\frac{Z}{1-\delta}-\frac{\delta}{1-\delta}\frac{I_m}{m}-S\biggr\|_{\infty}= \frac{1}{1-\delta}\|Z-S_{\delta}\|_{\infty}
\\
&
\nonumber
\leq \frac{1}{1-\delta}\|Z-S\|_{\infty}+ 
\frac{1}{1-\delta}\|S-S_{\delta}\|_{\infty}\leq \frac{1}{1-\delta}\|Z-S\|_{\infty}
+\frac{2\delta}{1-\delta}.
\end{eqnarray}
Using bounds (\ref{odin-odin}), (\ref{dva-dva}) along with the bound of Theorem \ref{th_main_rank},
we get the bound of the lemma.

\end{proof}

We will use the bound of Lemma \ref{mix} to control $\|\check \rho_{\delta}-\rho\|_p$ for 
$\rho\in {\mathcal S}_{r,m}.$
To this end, we need to bound $\|\hat Z-\rho\|_{\infty}$ using matrix Bernstein 
inequalities exactly as it was done in the proof of theorems \ref{minmaxthm1_upper}, \ref{minmaxthm3'''_upper} and \ref{minmaxthm3_Pauli_upper} (under assumptions of these theorems). 
Denote by $\bar \Delta$ such an upper bound on $\|\hat Z-\rho\|_{\infty}$ that holds with probability a least $1-(2m)^{-A}.$ 
Recall that $\bar \Delta\asymp \sigma_{\xi}m^{3/2}\sqrt{\frac{A\log (2m)}{n}}$
under the conditions of Theorem \ref{minmaxthm1_upper} and $\bar \Delta\asymp U m^{3/2}\sqrt{\frac{A\log (2m)}{n}}$
under the conditions of Theorem \ref{minmaxthm3'''_upper} (it is the same under the conditions of Theorem \ref{minmaxthm3_Pauli_upper} with $U=m^{-1/2}$).
Setting $\Delta = \bar \Delta \wedge 1,$
we get from the bound of Lemma \ref{mix} that 
$$
\|\check \rho_{\delta} - \rho\|_p\leq 
\min \biggl(2^{3/p+1}r^{1/p}\Bigl(\Delta+2\delta\Bigr), 2(1-\delta)^{1/p}\Bigl(\Delta+2\delta\Bigr)^{1-1/p}\biggr)+2\delta
$$
that holds with the same probability at least $1-(2m)^{-A}.$ Recall that we replace $\bar \Delta$ by $\Delta$ since 
the left hand side $\|\check \rho_{\delta} - \rho\|_p\leq 2$;  for the same reason, we can and do drop the ``exponential
parts" of matrix Bernstein bounds leaving in the definition of $\Delta$ only the ``Gaussian parts". 
For $\delta \lesssim \Delta,$ we get 
$$
\|\check \rho_{\delta} - \rho\|_p\lesssim  
\min (r^{1/p}\Delta, \Delta^{1-1/p}).
$$
Exactly as in the proof of theorems \ref{minmaxthm1_upper}, \ref{minmaxthm3'''_upper} and \ref{minmaxthm3_Pauli_upper}, this implies that bounds (\ref{minmaxthm1boundq_upper}), (\ref{minmaxthm1boundH_upper}), (\ref{minmaxthm1boundq_U'''_upper}), (\ref{minmaxthm1boundH_U'''_upper}), 
(\ref{minmaxthm1boundq_U_P}) and (\ref{minmaxthm1boundH_U_P}) hold for estimator $\check \rho_{\delta}.$
 
The bound on the Kullback-Leibler divergence $K(\rho\|\check \rho_{\delta})$ is an immediate consequence of 
the bound on $\|\check \rho_{\delta}-\rho\|_1$ and the next 
lemma that follows from Corollary 1 in \cite{Aud}.

\begin{lemma} 
\label{KL-lem}
Let $S_1,S_2\in {\mathcal S}_m$ be density matrices and let $\beta:=\lambda_{\min}(S_2)$ be the 
smallest eigenvalue of $S_2.$ Suppose that $\beta>0.$ 
Then 
$$
K(S_1\| S_2)\leq \|S_1-S_2\|_1\log\biggl(1+\frac{\|S_1-S_2\|_1}{2\beta}\biggr).
$$ 
\end{lemma} 
 
We apply Lemma \ref{KL-lem} to $S_1=\rho, S_2=\check \rho_{\delta},$ observing that 
$\check \rho_{\delta}\in {\mathcal S}_{m,\delta}$ and $\lambda_{\min}(\check \rho_{\delta})\geq \delta/m.$    
We then use the bound on $\|\check \rho_{\delta}-\rho\|_1$ to complete the proof of the bound on $K(\rho\|\check \rho_{\delta}).$

\end{proof}
 
We conclude this section with a simple result concerning the least squares estimator $\hat \rho$ defined by 
(\ref{LS}). It shows that the estimators $\hat \rho$ and $\check \rho$ are close in the Hilbert-Schmidt norm.
As a result, the bounds of the previous theorems could be applied to estimator $\hat \rho$ as well
(at least, under some additional assumptions).  

\begin{theorem}
\label{check-hat}
Under the assumption that i.i.d. design variables $X_1,\dots, X_n$ are sampled 
from the uniform distribution $\Pi$ in an orthonormal basis ${\mathcal E}=\{E_1,\dots, E_{m^2}\},$
the following bound holds with some constant $C>0$ for all $A\geq 1$ with 
probability at least $1-(2m^2)^{-A}:$
$$
\|\check \rho-\hat \rho\|_2 \leq C m\sqrt{\frac{A\log (2m)}{n}}.
$$
\end{theorem}

\begin{proof}
Note that the gradient (and subgradient) of convex function $S\mapsto \|S-Z\|_2^2$ is equal to $2(S-Z).$
By a necessary condition of minimum in convex minimization problem  (\ref{def_projection}), for $\check \rho 
= \pi_{{\mathcal S}_m}(\hat Z),$ $\hat Z-\check \rho$ should belong to the normal cone $N_{{\mathcal S}_m}(\check \rho)$
of the convex set ${\mathcal S}_m$ at point $\check \rho$ (see \cite{aubin}, Proposition 5, Chapter 4, Section 1). 
Since both $\check \rho, \hat \rho\in {\mathcal S}_m,$ this implies that 
\begin{equation}
\label{nec_min_1}
\langle \check \rho -\hat Z, \check \rho- \hat \rho\rangle\leq 0. 
\end{equation}
Similar analysis of convex optimization problem (\ref{LS}) shows that 
$$
\biggl\langle\frac{m^2}{n}\sum_{j=1}^n (\langle \hat \rho,X_j\rangle -Y_j)X_j, \check \rho -\hat \rho\biggl\rangle \geq 0,
$$
which could be rewritten as follows:
\begin{equation}
\label{nec_min_2}
\biggl\langle\frac{m^2}{n}\sum_{j=1}^n \langle \hat \rho,X_j\rangle X_j - \hat Z, \check \rho -\hat \rho\biggl\rangle \geq 0.
\end{equation} 
Subtracting (\ref{nec_min_2}) from (\ref{nec_min_1}) yields 
$$
\biggl\langle\check \rho-\frac{m^2}{n}\sum_{j=1}^n \langle \hat \rho,X_j\rangle X_j, \check \rho -\hat \rho\biggl\rangle \leq 0,
$$ 
implying that 
\begin{equation}
\label{check_hat}
\|\check \rho -\hat \rho\|_2^2 = \langle \check \rho - \hat \rho, \check \rho - \hat \rho \rangle\leq 
\biggl\langle \frac{m^2}{n}\sum_{j=1}^n \langle \hat \rho,X_j\rangle X_j - \hat \rho, \check \rho -\hat \rho\biggl\rangle.
\end{equation}
We will now write 
\footnote{Here we view the tensor product $A\otimes B$ of operators $A,B\in {\mathbb M}_m$ as an operator acting from the 
space ${\mathbb M}_m$ of $m\times m$ matrices equipped with Hilbert-Schmidt inner product $\langle \cdot, \cdot \rangle$
into itself as follows: $(A\otimes B)C=A \langle C, B\rangle.$ Let $\|\cdot \|_{{\rm op}}$ denote 
the operator norm of linear operators from ${\mathbb M}_m$ into itself, which corresponds to the $\|\cdot\|_{\infty}$
in the case of $m\times m$ matrices.} 
$$
\frac{m^2}{n}\sum_{j=1}^n \langle \hat \rho,X_j\rangle X_j - \hat \rho 
= \frac{m^2}{n}\sum_{j=1}^n \Bigl(\langle \hat \rho,X_j\rangle X_j - {\mathbb E}\langle \hat \rho, X\rangle X\Bigr)  
$$ 
$$
=m^2\biggl[\frac{1}{n}\sum_{j=1}^n (X_j \otimes X_j-{\mathbb E}(X\otimes X)\rangle\biggr]\hat \rho.
$$
It follows from 
(\ref{check_hat}) that 
$$
\|\check \rho -\hat \rho\|_2^2 \leq 
m^2\biggl\|\frac{1}{n}\sum_{j=1}^n X_j \otimes X_j - {\mathbb E}(X\otimes X)\biggr\|_{{\rm op}}\|\hat \rho\|_2
\|\check \rho -\hat \rho\|_2.
$$
Since $\|\hat \rho\|_2\leq 1,$ we get
\begin{equation}
\label{hat_check}
\|\check \rho -\hat \rho\|_2 \leq
m^2\biggl\|\frac{1}{n}\sum_{j=1}^n X_j \otimes X_j - {\mathbb E}(X\otimes X)\biggr\|_{{\rm op}}.
\end{equation}
It remains to control the operator norm in the right hand side for which we can again use matrix Bernstein 
inequality of Lemma \ref{Bernstein_standard} applying it to $V=X\otimes X-{\mathbb E}(X\otimes X).$
In this case, 
$$
\sigma^2=\|{\mathbb E}V^2\|_{{\rm op}} \leq \|{\mathbb E}(X\otimes X)^2\|_{{\rm op}} 
= \sup_{\|U\|_2\leq 1}{\mathbb E}\langle (X\otimes X)^2 U,U\rangle
=\sup_{\|U\|_2\leq 1}{\mathbb E}\langle (X\otimes X)U, (X\otimes X)U\rangle
$$ 
$$
= \sup_{\|U\|_2\leq 1}{\mathbb E}|\langle U,X\rangle|^2 \|X\|_2^2
\leq \sup_{\|U\|_2\leq 1}{\mathbb E}|\langle U, X\rangle|^2 = 
\sup_{\|U\|_2\leq 1}\frac{\|U\|_2^2}{m^2}= \frac{1}{m^2}
$$ 
and 
$$
\|V\|_{{\rm op}} \leq \|X\otimes X\|_{{\rm op}}+ {\mathbb E}\|X\otimes X\|_{{\rm op}}
=\|X\|_2^2 + {\mathbb E}\|X\|_2^2\leq 2.
$$ 
Bound (\ref{hat_check}) along with the bound  of Lemma \ref{Bernstein_standard} with $t=A\log (2m^2), A\geq 1$ yield the following 
inequality 
$$
\|\check \rho -\hat \rho\|_2 \lesssim 
m\sqrt{\frac{A\log (2m)}{n}} \bigvee m^2  \frac{A\log (2m)}{n}
$$
that holds with probability at least $1-(2m^2)^{-A}.$
Since $\|\check \rho -\hat \rho\|_2\leq 2,$ the second term $m^2  \frac{A\log (2m)}{n}$
in the right hand side could be dropped (if this term is dominant, the bound is $\gtrsim 1$). 
This completes the proof of the theorem.
\end{proof} 

Since $\|\check \rho-\hat \rho\|_{\infty}\leq \|\check \rho -\hat \rho\|_2,$ the bound of Theorem \ref{check-hat}
also holds for $\|\check \rho - \hat \rho\|_{\infty}.$ Combining this with the bound of Theorem \ref{minmaxthm3'''_upper} for 
$p=\infty,$
it is easy to conclude that under conditions of this theorem 
$$
\|\hat \rho - \rho\|_{\infty} \lesssim Um^{3/2}\sqrt{\frac{A\log (2m)}{n}}
$$
and that the last bound holds (with a proper choice of constant in relationship $\lesssim $) with probability 
at least $1-(2m)^{-A}.$ In view of Lemma \ref{rank_r}, this immediately implies that all the bounds of Theorem \ref{minmaxthm3'''_upper} also hold for the least squares estimator $\hat \rho.$ In a special case of Pauli measurements, this means that Theorem \ref{minmaxthm3_Pauli_upper} holds for the estimator $\hat \rho.$ Concerning Theorem 
\ref{minmaxthm1_upper''}, the same conclusion is true under the additional assumption that $\sigma_{\xi}\geq m^{-1/2}.$
Moreover, if $\hat \rho_{\delta}$ is the following modification of estimator $\hat \rho$
\begin{equation}
\label{LS-delta}
\hat \rho_{\delta} := {\rm argmin}_{S\in {\mathcal S}_{m,\delta}}\biggl[n^{-1}\sum_{j=1}^n (Y_j-\langle S,X_j\rangle)^2\biggr],
\end{equation} 
then the statements of theorems \ref{minmaxthm1_upper''}, \ref{minmaxthm3'''_upper''} and  \ref{minmaxthm3_Pauli_upper''} hold for the estimator $\hat \rho_{\delta}$ (in the case 
of Theorem \ref{minmaxthm1_upper''}, under the additional assumption that $\sigma_{\xi}\geq m^{-1/2}$).

\section{Proof of Theorem \ref{th:min-dist}}\label{last}

Recall that 
$$
\pi_{{\mathcal S}_m}(Z):= {\rm argmin}_{S\in {\mathcal S}_m}\|Z-S\|_2^2, Z\in {\mathbb H}_m
$$
defines the projection of $Z$ onto ${\mathcal S}_m.$
The mapping ${\mathbb H}_m\ni Z\mapsto \pi_{{\mathcal S}_m}(Z)\in {\mathcal S}_m$
possesses a couple of simple properties stated in the next proposition. Denote by ${\mathcal S}_m^d$
the set of all diagonal density matrices. 

\begin{proposition}
\label{simple}
\begin{enumerate}
\item \label{item:1} For all $m\times m$ unitary matrices $U,$ 

$$
\pi_{{\mathcal S}_m}(U^{-1}ZU)= U^{-1}\pi_{{\mathcal S}_m}(Z)U, Z\in {\mathbb H}_m.
$$

\item \label{item:2} If $D\in {\mathbb H}_m$ is a diagonal matrix, then $\pi_{{\mathcal S}_m}(D)\in {\mathcal S}_m^d.$
\end{enumerate}
\end{proposition} 

\begin{proof} To prove the first claim, note that, by the unitary invariance of the Hilbert--Schmidt norm,
$$
\|U^{-1}ZU-S\|_2^2 = \|U^{-1}(Z- USU^{-1})U\|_2^2 = \|Z-USU^{-1}\|_2^2.
$$
In addition, the mapping $S\mapsto USU^{-1}$ is a bijection from the set ${\mathcal S}_m$ onto 
itself. This immediately implies that 
$$
\pi_{{\mathcal S}_m}(U^{-1}ZU)= {\rm argmin}_{S\in {\mathcal S}_m}\|Z-USU^{-1}\|_2^2
= U^{-1}\pi_{{\mathcal S}_m}(Z)U.
$$

For an $m\times m$ matrix $A=(a_{ij})_{i,j=1}^m\in {\mathbb H}_m,$ let $A^d$ be the diagonal matrix with 
diagonal entries $a_{ii}, i=1,\dots, m.$ It is easy to see that if $A$ is a density matrix, then 
$A^d$ is also a density matrix.  Moreover, it is also obvious that, for a diagonal matrix $D,$ 
$$
\|D-A^d\|_2^2\leq \|D-A\|_2^2, A\in {\mathcal S}_m,
$$ 
with a strict inequality if $A$ is not diagonal. These observations immediately imply
the second claim.

\end{proof}

We will now state and prove a vector version of Theorem \ref{th:min-dist} in which the role 
of the set of density matrices ${\mathcal S}_m$ is played by 
the simplex 
$$\Delta_m:= \Bigl\{u=(u_1,\dots, u_m)\in {\mathbb R}^m: u_j\geq 0, \sum_{j=1}^m u_j=1\Bigr\}$$ 
in ${\mathbb R}^m$
(this is equivalent to considering the set of diagonal density matrices).  
We will then show that the matrix version of the problem reduces to the vector 
case. 

Define 
$$
\pi_{\Delta_m}(z) := {\rm argmin}_{u\in \Delta_m}\|z-u\|_{\ell_2^m}^2, z\in {\mathbb R}^m.
$$
Since the function $u\mapsto \|z-u\|_{\ell_2^m}^2$ is strictly convex 
and $\Delta_m$ is a compact convex set, such a minimizer exists and 
is unique. In other words, $\pi_{\Delta_m}(z)$ is the projection of the point 
$z\in {\mathbb R}^m$ onto simplex $\Delta_m$ (the closest point to $z$ in the set $\Delta_m$
with respect to the Euclidean $\ell_2^m$-distance).  The next lemma shows that the same point also minimizes the $\ell_{\infty}^m$-distance from $z$ to the simplex $\Delta_m.$

\begin{lemma}
\label{lemma_odin}
For all $z\in {\mathbb R}^m,$
$$
\|z-\pi_{\Delta_m}(z)\|_{\ell_{\infty}^m} = \min_{v\in \Delta_m}\|z-v\|_{\ell_{\infty}^m}.
$$
\end{lemma} 

\begin{proof}
Without loss of generality, assume that $z=(z_1,\dots, z_m)\in {\mathbb R}^m$ is a point with 
$z_1\geq \dots \geq z_m.$ Denote 
$$
\bar z_j := \frac{z_1+\dots+z_j}{j}, j=1,\dots, m.
$$
Clearly, $\bar z_1=z_1$ and $\bar z_j\geq z_j, j=1,\dots, m.$
Let 
$$
k:= \max\biggl\{ j\leq m: \bar z_j\leq z_j + \frac{1}{j}\biggr\}.
$$
Note that if $k>1,$ then, for all $j<k,$ $\bar z_j\leq z_j+\frac{1}{j}.$
Indeed, 
$$
\bar z_j =\frac{k\bar z_k -\sum_{i=j+1}^k z_i}{j}\leq \frac{kz_k+1 -(k-j)z_k}{j}
=\frac{jz_k+1}{j}=z_k+\frac{1}{j}\leq z_j+\frac{1}{j}.
$$
On the other hand, if $k<m,$ then $\bar z_k>z_{k+1}+\frac{1}{k}.$ Indeed, if $\bar z_k\leq z_{k+1}+\frac{1}{k},$ then 
$$
\bar z_{k+1} = \frac{k\bar z_k+z_{k+1}}{k+1} \leq \frac{kz_{k+1}+1+ z_{k+1}}{k+1}=z_{k+1}+\frac{1}{k+1},
$$
which would contradict the definition of $k.$

Let 
$\lambda= (\lambda_1,\dots, \lambda_m),$ where 
$\lambda_j= z_j-\bar z_k+\frac{1}{k}$ for $j=1,\dots, k$ and $\lambda_j=0$
for $j=k+1,\dots, m.$ Since $\bar z_k\leq z_k+\frac{1}{k}\leq z_j+\frac{1}{k}$ for all $j\leq k,$
we have $\lambda_j\geq 0, j=1,\dots, m$ and 
$$
\sum_{j=1}^m \lambda_j = \sum_{j=1}^k \biggl(z_j-\bar z_k+\frac{1}{k}\biggr) = \sum_{j=1}^k z_j -k\bar z_k+1=1.
$$
Thus, $\lambda \in \Delta_m.$ It turns out that $\pi_{\Delta_m}(z)=\lambda.$ 
\footnote{The computation of the projection onto a simplex occurs in many applications and 
has been studied before: see, e.g., \cite{michelot1986finite} and \cite{Shalev}. See also \cite{chen2011projection}, where an explicit expression for the projection was 
derived. For completeness, we provide our version of the proof below.} 
To prove this it is enough to show that $z-\lambda \in N_{\Delta_m}(\lambda),$
where 
$$
N_{\Delta_m}(\lambda):= \{u\in {\mathbb R}^m: \langle u, v-\lambda\rangle \leq 0, v\in \Delta_m\}
$$
is the normal cone of the convex set $\Delta_m$ at point $\lambda$ (see, e.g., \cite{aubin}, Proposition 5, Chapter 4, Section 1). Let $t:= \bar z_k-\frac{1}{k}.$ Clearly, we have 
$z_{k+1}<t\leq z_k$ if $k<m$ and $t\leq z_m$ if $k=m.$
For $k=m,$ $z-\lambda= (t,\dots, t)$ and 
$$
\langle z-\lambda, v-\lambda\rangle=\sum_{i=1}^m t(v_i-\lambda_i)= 
t \biggl(\sum_{i=1}^m v_i-\sum_{i=1}^m \lambda_i\biggr)=0
$$ 
since $v,\lambda \in \Delta_m.$
For $k<m,$ note that 
$$
z-\lambda = (t, \dots t, z_{k+1}, \dots, z_m)
$$
and, for $v\in \Delta_m,$ 
$$
\langle z-\lambda, v-\lambda\rangle = 
\sum_{i=1}^k t(v_i-\lambda_i) + \sum_{i=k+1}^m z_i v_i.
$$
Using the facts that $\sum_{i=1}^m v_i=1$ and $\sum_{i=1}^k \lambda_i=1,$
we get
$$
\langle z-\lambda, v-\lambda\rangle =t\biggl(\sum_{i=1}^k v_i-\sum_{i=1}^k \lambda_i\biggr)
+\sum_{i=k+1}^m z_i v_i 
$$
$$
=
-t\sum_{i=k+1}^m v_i+ \sum_{i=k+1}^m z_i v_i
= \sum_{i=k+1}^m (z_i-t)v_i \leq 0,
$$
where we also used that, for all $i=k+1,\dots, m,$ $z_i-t\leq z_{k+1}-t\leq 0$
and $v_i\geq 0.$ Thus, 
$z-\lambda\in N_{\Delta_m}(\lambda)$ and, by the uniqueness of the minimum, $\lambda = \pi_{\Delta_m}(z).$

Note that 
$$
\|z-\lambda\|_{\ell_{\infty}^m}= \max(|t|, |z_{k+1}|, \dots, |z_m|).
$$
For any $v\in \Delta_m,$
$$
t=\bar z_k -\frac{1}{k}= \frac{1}{k}\sum_{i=1}^k z_i -\frac{1}{k}\sum_{i=1}^m v_i 
\leq \frac{1}{k}\sum_{i=1}^k z_i -\frac{1}{k}\sum_{i=1}^k v_i = 
\frac{1}{k}\sum_{i=1}^k (z_i-v_i)\leq \|z-v\|_{\ell_{\infty}^m}.
$$
On the other hand, 
$$
z_m \geq z_m -v_m \geq -\|z-v\|_{\ell_{\infty}^m}.
$$
Since 
$$
t=\bar z_k -\frac{1}{k}\geq z_{k+1}\geq \dots \geq z_m,
$$
we conclude that, for all $v\in \Delta_m,$ 
$$
\|z-\lambda\|_{\ell_{\infty}^m}\leq \|z-v\|_{\ell_{\infty}^m}.
$$

\end{proof}

We now turn to the proof of Theorem \ref{th:min-dist}. 

\begin{proof} 
Any matrix $Z\in {\mathbb H}_m$ admits spectral representation $Z=U^{-1}DU,$ where
$D$ is the diagonal matrix with real entries $d_1, \dots, d_m$ on the diagonal and 
$U$ is a unitary $m\times m$ matrix. Let $d=(d_1,\dots, d_m)\in {\mathbb R}^m.$
Given $v=(v_1,\dots, v_m)\in \Delta_m,$ the diagonal matrix $V$ with entries $v_1,\dots, v_m$
is a density matrix. This defines a bijection $\Delta_m\ni v\mapsto V=J(v)$ between the simplex 
$\Delta_m$ and the set ${\mathcal S}_m^d$ of all diagonal $m\times m$ density matrices. Moreover,
$J$ is an isometry of $\Delta_m$ and ${\mathcal S}_m^d:$
$\|J(v)-J(u)\|_2^2=\|u-v\|_{\ell_2^m}^2, u,v\in \Delta_m.$

We will now prove the following lemma.

\begin{lemma}
\label{lemma_dva}
Let $Z=U^{-1}DU$ with a unitary $m\times m$ matrix $U$ and diagonal matrix $D$ with 
$d=(d_1,\dots, d_m)\in {\mathbb R}^m$ being the vector of its diagonal entries.
Then 
$$
\pi_{{\mathcal S}_m}(Z)=U^{-1}J(\pi_{\Delta_m}(d))U.
$$
\end{lemma}

\begin{proof} This is an immediate consequence of Proposition \ref{simple}
and the following simple fact: 
$$
{\rm argmin}_{A\in {\mathcal S}_m^d}\|D-A\|_2^2
=J\biggl({\rm argmin}_{v\in \Delta_m}\|J(d)-J(v)\|_2^2\biggr)
$$
$$
J\biggl({\rm argmin}_{v\in \Delta_m}\|d-v\|_{\ell_2^m}^2\biggr)= J(\pi_{\Delta_m}(d)).
$$

\end{proof}

To complete the proof of Theorem \ref{th:min-dist}, observe that,
In view of lemmas \ref{lemma_odin}, \ref{lemma_dva},   
$$
\|Z-\pi_{{\mathcal S}_m}(Z)\|_{\infty}= \|U^{-1}(J(d)-J(\pi_{\Delta_m}(d)))U\|_{\infty}
$$
$$
=\|J(d)-J(\pi_{\Delta_m}(d))\|_{\infty} =\|d-\pi_{\Delta_m}(d)\|_{\ell_{\infty}^m}
= \inf_{v\in \Delta_m} \|d-v\|_{\ell_{\infty}^m}.
$$
Without loss of generality, assume that $d_1\geq \dots \geq d_m.$ Let 
$S\in {\mathcal S}_m$ be a density matrix with eigenvalues $v_1\geq \dots \geq v_m.$
Clearly, $v=(v_1,\dots, v_m)\in \Delta_m.$ Therefore,
$$
\|Z-\pi_{{\mathcal S}_m}(Z)\|_{\infty}\leq \|d-v\|_{\infty}
\leq \|Z-S\|_{\infty},
$$
where to get the last bound we used Weyl's perturbation inequality (see \cite{Bhatia}, Corollary III.2.6). 

\end{proof}

\section{Comments on computational aspects of the problem}

An advantage of minimal distance estimator $\check{\rho}=\pi_{\mathcal{S}_m}(\hat{Z})$ is the simplicity of its computational implementation. The computation of the matrix 
$\hat{Z}=\frac{m^2}{n}\sum_{i=1}^nY_iX_i$ requires $O(nm^2)$ operations.
It is followed by an eigen-decomposition of $Z$ that requires $O(m^3)$ operations(see \cite{golub2012matrix}); there exist efficient software packages designed for this kind of tasks, for instance, LINPACK and PROPACK, etc.). As it is shown in the previous section, the problem of computing 
$\pi_{\mathcal{S}_m}(\hat{Z})$ then reduces to projecting of the vector of eigenvalues of $Z$ arranged in a non-increasing order onto the simplex $\Delta_m.$ The last problem has been studied in the 
literature (see \cite{michelot1986finite}, \cite{Shalev}, \cite{chen2011projection}) 
and it has an explicit solution of computational 
complexity proportional to $m$ (see the proof of Lemma \ref{lemma_odin}). 
Thus, the computational implementation of the minimal distance estimator $\check \rho$
requires $O((n+m)m^2)$ operations. 

The matrix version of LASSO estimator for density matrices 
is equivalent to solving the following optimization problem
\begin{equation}\label{lse}
\hat{\rho}:=\underset{S\in\mathcal{S}_m}{\arg\min} \frac{1}{n}\sum_{i=1}^n\Big(Y_i-\big<S,X_i\big>\Big)^2
\end{equation}
that results in the least squares estimator. 
Clearly, there is no explicit solution for this optimization problem and it is usually solved by iterative algorithms.
For example, a well know iterative singular value thresholding (SVT) algorithm was proposed in \cite{cai2010singular}, and also implemented in quantum compressed sensing in \cite{flammia2012quantum}. The main idea
is that (\ref{lse}) is equivalent to the following optimization problem: for any $\tau>0$,
\begin{equation*}
 \hat{\rho}:=\underset{S\in\mathcal{S}_m,Z\in\mathbb{H}_m, S=Z}{\arg\min} \frac{m^2}{n}\sum_{i=1}^n\Big(Y_i-\big<Z,X_i\big>\Big)^2+\tau\|S-Z\|_2^2.
\end{equation*}
The proposed algorithm updates $Z$ and $S$ alternatively, with the only constraint for $S$ being that $S\in\mathcal{S}_m$. Therefore, 
the main ingredient of SVT is the following iterative updating rule (with initial $Z_0=0$): for $k=1,2,\ldots$,
\begin{equation}
\label{iteralg}
\begin{cases}
S_k=\pi_{\mathcal{S}_m}(Z_{k-1})\\
Z_k=S_{k}+\delta_k\big(\hat{Z}-\frac{m^2}{n}\sum_{i=1}^n\big<S_k,X_i\big>X_i\big)
\end{cases}
\end{equation}
with certain pre-determined step sizes $\delta_k> 0$. The algorithm terminates at some step $k=N$ and outputs $S_N\in\mathcal{S}_m$ when $\|S_N-S_{N-1}\|_2\leq \epsilon$ 
for some numerical threshold $\epsilon>0$. 
It is clear that the minimal distance estimator $\check{\rho}$ can be produced by the above algorithm with one iteration and the initialization $Z_0=\hat{Z}, \delta_1=0.$
When the number of qubits $k$ is not small (for instance, about $20$) and the dimension $m$ is very large, the iterative algorithm (\ref{iteralg}) is much more computationally expensive than the 
algorithm for the minimal distance estimator (since every iteration requires the eigen-decomposition of a high dimensional matrix).

\bibliographystyle{abbrv}
\bibliography{refer_1}

\begin{thebibliography}{10}

\bibitem{Alquier}
P.~Alquier, C.~Butucea, M.~Hebiri, K.~Meziani, and T.~Morimae.
\newblock Rank penalized estimation of a quantum system.
\newblock {\em Physical Reviews A}, 88:032113, 2013.

\bibitem{aubin}
J.-P. Aubin and I.~Ekeland.
\newblock {\em Applied Nonlinear Analysis}.
\newblock Courier Corporation, 2006.

\bibitem{Aud}
K.~Audenaert and J.~Eisert.
\newblock Continuity bounds on the quantum relative entropy - ii.
\newblock {\em Journal of Mathematical Physics}, 52(112201), 2011.

\bibitem{Bhatia}
R.~Bhatia.
\newblock {\em Matrix Analisis}.
\newblock Springer, 1997.

\bibitem{cai2010singular}
J.-F. Cai, E.~J. Cand{\`e}s, and Z.~Shen.
\newblock A singular value thresholding algorithm for matrix completion.
\newblock {\em SIAM Journal on Optimization}, 20(4):1956--1982, 2010.

\bibitem{candes2011tight}
E.~J. Cand\`es and Y.~Plan.
\newblock Tight oracle inequalities for low-rank matrix recovery from a minimal
  number of noisy random measurements.
\newblock {\em IEEE Transactions on Information Theory}, 57(4):2342--2359,
  2011.

\bibitem{candes2010power}
E.~J. Cand\`es and T.~Tao.
\newblock The power of convex relaxation: Near-optimal matrix completion.
\newblock {\em IEEE Transactions on Information Theory}, 56(5):2053--2080,
  2010.

\bibitem{chen2011projection}
Y.~Chen and X.~Ye.
\newblock Projection onto a simplex.
\newblock {\em arXiv preprint arXiv:1101.6081}, 2011.

\bibitem{flammia2012quantum}
S.~T. Flammia, D.~Gross, Y.-K. Liu, and J.~Eisert.
\newblock Quantum tomography via compressed sensing: error bounds, sample
  complexity and efficient estimators.
\newblock {\em New Journal of Physics}, 14(9):095022, 2012.

\bibitem{golub2012matrix}
G.~H. Golub and C.~F. Van~Loan.
\newblock {\em Matrix computations}, volume~3.
\newblock JHU Press, 2012.

\bibitem{gross2011recovering}
D.~Gross.
\newblock Recovering low-rank matrices from few coefficients in any basis.
\newblock {\em IEEE Transactions on Information Theory}, 57(3):1548--1566,
  2011.

\bibitem{gross2010quantum}
D.~Gross, Y.-K. Liu, S.~T. Flammia, S.~Becker, and J.~Eisert.
\newblock Quantum state tomography via compressed sensing.
\newblock {\em Physical Review Letters}, 105(15):150401, 2010.

\bibitem{Klauck2007}
H.~Klauck, A.~Nayak, A.~Ta-Shma, and D.~Zuckerman.
\newblock Interaction in quantum communication.
\newblock {\em IEEE Transactions on Information Theory}, 53(6):1970--1982,
  2007.

\bibitem{klopp2014noisy}
O.~Klopp.
\newblock Noisy low-rank matrix completion with general sampling distribution.
\newblock {\em Bernoulli}, 20(1):282--303, 2014.

\bibitem{Koltchinskii2011oracle}
V.~Koltchinskii.
\newblock {\em Oracle Inequalities in Empirical Risk Minimization and Sparse
  Recovery Problems: {\'E}cole d'{\'E}t{\'e} de Probabilit{\'e}s de Saint-Flour
  XXXVIII-2008}.
\newblock Springer, 2011.

\bibitem{Koltchinskii2011neumann}
V.~Koltchinskii.
\newblock von {Neumann} entropy penalization and low-rank matrix estimation.
\newblock {\em The Annals of Statistics}, 39(6):2936--2973, 2011.

\bibitem{Koltchinskii2013remark}
V.~Koltchinskii.
\newblock A remark on low rank matrix recovery and noncommutative {Bernstein}
  type inequalities.
\newblock In {\em From Probability to Statistics and Back: High-Dimensional
  Models and Processes--A Festschrift in Honor of Jon A. Wellner}, pages
  213--226. Institute of Mathematical Statistics, 2013.

\bibitem{Koltchinskii2013sharp}
V.~Koltchinskii.
\newblock Sharp oracle inequalities in low rank estimation.
\newblock In {\em Empirical Inference}, pages 217--230. Springer, 2013.

\bibitem{Koltchinskii2011nuclear}
V.~Koltchinskii, K.~Lounici, and A.~B. Tsybakov.
\newblock Nuclear-norm penalization and optimal rates for noisy low-rank matrix
  completion.
\newblock {\em The Annals of Statistics}, 39(5):2302--2329, 2011.

\bibitem{Koltch_Xia_15}
V.~Koltchinskii and D.~Xia.
\newblock Optimal estimation of low rank density matrices.
\newblock {\em Journal of Machine Learning Research}, 16(Sep):1757--1792, 2015.

\bibitem{liu2011universal}
Y.-K. Liu.
\newblock Universal low-rank matrix recovery from {Pauli} measurements.
\newblock In {\em Advances in Neural Information Processing Systems}, pages
  1638--1646, 2011.

\bibitem{Lounici_arxiv}
K.~Lounici.
\newblock Optimal spectral norm rates for noisy low-rank matrix completion.
\newblock arxiv:1110.5346, 2011.

\bibitem{ma2013volume}
Z.~Ma and Y.~Wu.
\newblock Volume ratio, sparsity, and minimaxity under unitarily invariant
  norms.
\newblock {\em IEEE Transactions on Information Theory}, 61(12):6939--6956,
  2015.

\bibitem{michelot1986finite}
C.~Michelot.
\newblock A finite algorithm for finding the projection of a point onto the
  canonical simplex of $\mathbb{R}^n$.
\newblock {\em Journal of Optimization Theory and Applications},
  50(1):195--200, 1986.

\bibitem{negahban}
S.~Negahban and M.~J. Wainwright.
\newblock Estimation of (near) low-rank matrices with noise and
  high-dimensional scaling.
\newblock {\em The Annals of Statistics}, 39(2):1069--1097, 2011.

\bibitem{Nielsen2000}
M.~Nielsen and I.~Chuang.
\newblock {\em Quantum Computation and Quantum Information.}
\newblock Cambridge University Press, 2000.

\bibitem{Shalev}
S.~Shalev-Shwartz and Y.~Singer.
\newblock Efficient learning of label ranking by soft projections onto
  polyhedra.
\newblock {\em Journal of Machine Learning Research}, 7:1567--1599, 2006.

\bibitem{tropp2012user}
J.~A. Tropp.
\newblock User-friendly tail bounds for sums of random matrices.
\newblock {\em Foundations of Computational Mathematics}, 12(4):389--434, 2012.

\end{thebibliography}

\end{document}